\documentclass[preprint,12pt]{colt2024} %

\title[Robust Estimation under Local and Global Adversarial Corruptions]{Robust Distribution Learning with\\ Local and Global Adversarial Corruptions}
\usepackage{times}

\coltauthor{\Name{Sloan Nietert} \Email{nietert@cs.cornell.edu}\\
\Name{Ziv Goldfeld} \Email{goldfeld@cornell.edu}\\
\Name{Soroosh Shafiee} \Email{shafiee@cornell.edu}\\
\addr Cornell University}

\usepackage{colt_macros}

\begin{document}

\maketitle

\begin{abstract}%

We consider learning in an adversarial environment, where an $\eps$-fraction of samples from a distribution $P$ are arbitrarily modified (\emph{global} corruptions) and the remaining perturbations have average magnitude bounded by $\rho$ (\emph{local} corruptions). Given access to $n$ such corrupted samples, we seek a computationally efficient estimator $\hat{P}_n$ that  minimizes the Wasserstein distance $\Wone(\hat{P}_n,P)$. In fact, we attack the fine-grained task of minimizing $\Wone(\Pi_\sharp  \hat{P}_n, \Pi_\sharp P)$ for all orthogonal projections $\Pi \in \R^{d \times d}$, with performance scaling with $\mathrm{rank}(\Pi) = k$. This allows us to account simultaneously for mean estimation ($k=1$), distribution estimation ($k=d$), as well as the settings interpolating between these two extremes. We characterize the optimal population-limit risk for this task and then develop an efficient finite-sample algorithm with error bounded by $\sqrt{\eps k} + \rho + \tilde{O}(d\sqrt{k}n^{-1/(k \lor 2)})$ when $P$ has bounded covariance. 
This guarantee holds uniformly in $k$ and is minimax optimal up to the sub-optimality of the plug-in estimator when $\rho = \eps = 0$.
Our efficient procedure relies on a novel trace norm approximation of an ideal yet intractable 2-Wasserstein projection estimator. We apply this algorithm to robust stochastic optimization, and, in the process, uncover a new method for overcoming the curse of dimensionality in Wasserstein distributionally robust optimization.\footnote{Accepted for presentation at the
Conference on Learning Theory (COLT) 2024.}
  
\end{abstract}


\section{Introduction}

In robust statistics and adversarial machine learning, estimation and decision-making are treated as a two-player game between the learner and a budget-constrained adversary. Through this lens, researchers have developed learning algorithms with strong guarantees despite adversarial corruptions. For example, Huber's $\eps$-contamination model in classical robust statistics \citep{huber64} and the total variation (TV) $\eps$-contamination model \citep{donoho88} give the adversary an $\eps$ fraction of data to arbitrarily and globally corrupt. Popularized recently in the setting of adversarial training \citep{sinha2018certifying}, Wasserstein corruption models permit all of the data to be locally perturbed, bounding the average perturbation size by some radius $\rho \geq 0$. Recall that the $p$-Wasserstein distance is defined between distributions $P,Q$ by
\begin{equation*}
    \Wp(P,Q) \defeq \inf_{\pi \in \Pi(P,Q)} \E_{(X,Y) \sim \pi} \big[\|X - Y\|^p\big]^{\frac{1}{p}},
\end{equation*}
    where $\Pi(P,Q)$ is the set of their couplings. This metric naturally lifts the Euclidean geometry of $\R^d$ to the space of distributions $\cP(\R^d)$ with finite $p$-th absolute moments.
\medskip

Ideally, a corruption model should be flexible enough to capture multiple types of data contamination. Towards this goal, we investigate learning under combined TV and Wasserstein adversarial corruptions, recently introduced in the setting of distributionally robust optimization (DRO) \citep{nietert2023outlier}. Formally, we consider learning where clean samples $X_1, \dots, X_n \sim P$ are arbitrarily perturbed to obtain $\{\tilde{X}_i\}_{i=1}^n$ such that $\sum_{i \in S}\|\tilde{X}_i - X_i\| \leq \rho$, where $S \subseteq [n]$ with $|S| \geq (1-\eps)n$. 
Denoting the clean and corrupted empirical measures by $P_n$ and $\tilde{P}_n$, respectively, this corruption model is characterized by an outlier-robust variant of the Wasserstein distance defined in \eqref{eq:RWp} ahead, whereby $\RWone(P_n,\tilde{P}_n) \leq \rho$. We ask:
\begin{quote}
    \emph{How can we learn effectively and efficiently with both local and global adversarial corruptions?}
\end{quote}

\medskip

Under this combined model, we seek an estimate $\hat{P}_n$ that can approximate $P$ in a variety of downstream applications. In particular, we explore the distribution learning task of recovering $P$ under $\Wone$ itself. Since the sample complexity and risk bounds associated with standard $\Wone$ suffer a curse of dimensionality, we focus on the fine grained-goal of estimating $k$-dimensional projections of $P$, with performance scaling with $k$. Quantitatively, we seek $\hat{P}_n$ such that the \emph{$k$-dimensional max-sliced Wasserstein distance}
\begin{equation*}
    \Wonek(\hat{P}_n,P) \defeq \sup_{\substack{U \in \R^{k \times d}\\ UU^\top = I_k}} \Wone(U_\# \hat{P}, U_\# P) = \sup_{\substack{U \in \R^{k \times d}, f \in \Lip_1(\R^k)\\ UU^\top = I_k}} \E_{\hat{P}_n}[f(UX)] - \E_{P}[f(UX)]
\end{equation*}
is appropriately small for all $k \in [d]$. Since our approach cleanly addresses all slicing dimensions simultaneously, we focus on providing bounds which are uniform in $k$. By doing so, we account not only for the said distribution estimation task ($k=d$), but also mean estimation ($k=1$).
\medskip

While this task is relatively straightforward under TV corruption alone (we show in \cref{ssec:istribution-learning-extensions} that a standard iterative filtering algorithm \citep{diakonikolas2016robust} suffices), and immediate under Wasserstein corruption alone, where the corrupted distribution $\tilde{P}_n$ is itself minimax optimal, the combined model requires a new algorithmic approach and analysis to obtain suitable risk bounds. Eventually, we revisit the Wasserstein DRO setting that introduced in \citep{nietert2023outlier}. Here, the steps we take to employ our estimate lead to a new perspective on generalization and radius selection when employing Wasserstein ambiguity sets for distributionally robust stochastic optimization.

\subsection{Our Results}

Assuming that $\Sigma_P \preceq I_d$ and $n = \Omega(d \log(d)/\eps)$, we propose an algorithm \FilterSimple{} (\cref{alg:W2-project}) which, given the corrupted data $\tilde{P}_n$, efficiently computes an estimate 
$\hat{P}_n$ such that
\begin{equation}
\label{eq:W2Project-bd}
    \Wonek(\hat{P}_n,P) \lesssim \sqrt{k\eps}+ \rho + \tilde{O}\bigl(\sqrt{d}kn^{-\frac{1}{k \lor 2}}\bigr),
\end{equation}
for all $k \in [d]$. The final term is a bound on the sampling error $\max_k \Wonek(P,P_n)$, while the first two terms are minimax-optimal for large sample sizes. Even with no corruptions ($\rho = \eps = 0$), it is known that error $\Wonek(\hat{P}_n,P) \gtrsim c_d n^{-1/k} + \sqrt{d/n}$ is unavoidable \citep{niles2022estimation}, so our finite-sample guarantee is near-optimal, up to the sub-optimality of the plug-in estimator $P_n$ for learning $P$ under $\Wonek$.\medskip

Our algorithm serves as a tractable proxy for the minimum distance estimate
\begin{equation*}
    \hat{P}_{\mathrm{MDE}} = \argmin_{Q \leq \frac{1}{1-\eps}\tilde{P}_n} \Wtwo(Q,\Gcov),
\end{equation*}
where $Q$ ranges over all distributions obtained by deleting an $\eps$-fraction of mass from $\tilde{P}_n$ and renormalizing, and $\Wtwo(Q,\Gcov) = \inf_{R :\Sigma_R \preceq I_d} \Wtwo(Q,R)$. In particular, we approximate $\Wtwo(Q,\Gcov) \approx \tr(\Sigma_Q - I_d)_+$, which lends itself to efficient implementation using spectral decomposition. The resulting algorithm, \FilterSimple{}, employs a multi-directional filtering procedure that generalizes the standard iterative filtering algorithm for robust mean estimation under TV $\eps$-corruptions \citep{diakonikolas2017being}. In the infinite-sample population-limit, we prove tight risk bounds for broader class of probability measures (e.g., those which are sub-Gaussian or log-concave) using a family of related minimum distance estimators.
\medskip

Given such an algorithm, we then explore applications to robust stochastic optimization. Suppose that we have an estimate $\hat{P}$ of known quality $\Wonek(\hat{P},P) \leq \tau$, perhaps from the procedure above. Given a family of Lipschitz loss functions $\cL$ which operate on $k$-dimensional linear features (e.g. $k$-variate linear regression) we prove that the Wasserstein DRO estimate
\begin{equation}
\label{eq:WDRO-intro}
    \hat{\ell} = \argmin_{\ell \in \cL} \sup_{Q:\Wone(Q,\hat{P}_n) \lesssim \tau} \E_Q[\ell]
\end{equation}
achieves risk bounds typically associated with the sliced-Wasserstein DRO problem
\begin{equation*}
    \min_{\ell \in \cL} \sup_{Q:\Wonek(Q,\hat{P}_n) \lesssim \tau} \E_Q[\ell],
\end{equation*}
even though $P$ need not belong to the Wasserstein ambiguity set in \eqref{eq:WDRO-intro}. In particular, we prove that $\hat{\ell}$ satisfies the excess risk bound $\E_P[\hat{\ell}] - \E_P[\ell_\star] \lesssim \|\ell_\star\|_{\Lip} \tau$, where $\ell_\star = \argmin_{\ell \in \cL} \E_P[\ell]$. Plugging in the algorithmic results above improves upon existing results for outlier-robust WDRO, exhibiting tight dependence on $k$ and with sampling error scaling as $n^{-1/k}$ rather than $n^{-1/d}$. This risk bound would be immediate if the $\Wone$ ball above was replaced with a $\Wonek$ ball. The fact that this is not necessary is essential for computational tractability, and provides a new framework for avoiding the curse of dimensionality (CoD) in Wasserstein DRO. We note that this result is new even when $\eps = 0$ and $\rho > 0$ is taken to model only stochastic sampling error. Previous results on avoiding the CoD required $k=1$ or involved significantly more complicated analysis. In particular, for the rank-one linear structure with $k=1$, including univariate linear regression/classification, bounds of order $O(n^{-1/2})$ were established in \citep{shafieezadeh2019regularization,chen2018robust,olea2022generalization,wu2022generalization}. On the other hand, \citet{gao2022finite} relaxes the rank-one structural assumption and achieves $O(n^{-1/2})$ bounds as long as the data generating distribution satisfies certain transport inequalities. Nonetheless, the required assumptions are not easily verifiable.

\subsection{Related Work}

\paragraph{Robust statistics.}
Learning from data under TV $\eps$-corruptions, a staple of classical robust statistics, dates back to \cite{huber64}. Various robust and sample-efficient estimators, particularly for mean and scale parameters, have been developed in the robust statistics community; see \cite{ronchetti2009robust} for a comprehensive survey. In computational learning theory, older work has explored probably approximately correct (PAC) learning framework with adversarially corrupted labels \citep{angluin1988learning, bshouty2002pac}. Recently, \citet{zhu2019resilience}, significantly expanding the celebrated results of \citet{donoho88}, developed a unified statistical framework for robust statistics based on minimum distance estimation and a generalized resilience quantity, providing sharp population-limit and strong finite-sample guarantees for tasks including mean and covariance estimation. Learning under $\Wone$ corruptions is considered in \citet{zhu2019resilience} and \citet{chao2023statistical}, but our distribution estimation task is trivial under local corruptions alone, and, as such, is not considered by these works.

Over the past decade, the focus in the computer science community has shifted to the high-dimensional setting, where they have developed computationally efficient estimators achieving optimal estimation rates for many problems
\citep{diakonikolas2016robust,cheng2019high, diakonikolas2023algorithmic}. Many such works involve filtering the corrupted dataset to shrink the eigenvalues of its empirical covariance matrix, and our algorithm is also of this flavor. Most related to our work is a multi-directional filtering subroutine for robust Gaussian mean estimation \citep{diakonikolas2018robustly}, which identifies subspaces of $\R^d$ where the eigenvalues of the empirical covariance matrix are large. 
While qualitatively similar, their algorithm employs an expensive robust mean estimation step along a lower-dimensional subspace at each iteration, while our approach simply uses the empirical mean (see Remark~\ref{remark:comp-to-gaussian-mean-estimation} for further discussion).

\paragraph{Robust optimal transport.} The robust optimal transport (OT) literature has a close connection with unbalanced OT theory, which deals with transportation problems between measures of different mass. Unbalanced OT problems involve $f$-divergences that account for differences in mass, which can appear either in the constraints \citep{balaji2020} or in the objective function as regularizers \citep{piccoli2014, chizat2018, liero2018, schmitzer2019, hanin1992}. The constraint version is usually more difficult to solve, whereas primal-dual type algorithms have been developed to solve the regularized version \citep{mukherjee2021,chizat2018scaling,fatras21a,fukunaga2021,le2021,nath2020}.
An alternative approach to model robustness in OT is through partial OT problems, where only a fraction of mass needs to be transported \citep{caffarelli2010, figalli2010, nietert2023robust, chapel2020}. Partial OT has been previously used in the context of DRO problems; however, it was introduced to address stochastic programs with side information~\citep{esteban2022distributionally}.

\paragraph{Sliced optimal transport.}
Max-sliced optimal transport, as used to define $\Wonek$, is also known as $k$-dimensional optimal transport \citep{niles2022estimation} and projection-robust optimal transport \citep{lin2020projection, lin2021projection}. In general, $\Wonek$ defines a metric on the space of $d$-dimensional distributions by measuring discrepancy between $k$-dimensional projections thereof. The structural, statistical, and computational properties of $\Wonek$ are well-studied in \citep{lin2021projection,niles2022estimation,nietert2022sliced,bartl2022structure, nadjahi2020statistical,boedihardjo2024sharp}, with the tightest results established for $k=1$. Max-sliced OT has been used in the context of DRO problems; however, it was introduced for rank-one linear structures~\citep{olea2022generalization}.

\paragraph{Distributionally robust optimization.}
Wasserstein distributionally robust optimization has emerged as a powerful modeling approach for addressing uncertainty in the data generating distribution. 
In this approach, the ambiguity set around the empirical distribution is constructed by the Wasserstein distance.
Modern convex approach, leveraging duality theory \citep{mohajerin2018data,blanchet2019quantifying,gao2023distributionally}, has led to significant computational advantages. 
Despite its computational success, some studies have raised concerns about the sensitivity of the standard DRO formulations to outliers \citep{hashimoto2018fairness,hu2018does,zhu2019resilience}.
To address potential overfitting to outliers, \citet{zhai2021doro} propose a refined risk function based on a family of $f$-divergences. Nevertheless, this approach is not robust to local perturbations, and the risk bounds require a moment condition to hold uniformly over $\Theta$. Another related work in \citep{bennouna2022holistic,bennouna2023certified} constructs the ambiguity set using an $f$-divergence for statistical errors and the Prokhorov distance for outliers. This provides computational efficiency and statistical reliability but lacks analysis of minimax optimality and robustness to Huber contamination. Furthermore, \citet{nietert2023outlier} constructs the ambiguity set using the robust Wasserstein distance introduced in~\citep{nietert2021outlier}. We revisit this setting in Section~\ref{sec:robust:stochastic}.

\subsection{Notation and Preliminaries} 
Let $\|\cdot\|$ denote the Euclidean norm on $\R^d$ and define $\unitsph \defeq \{ x \in \R^d : \|x\| = 1\}$. We write $\cP(\R^d)$ for the family of Borel probability measures on $\R^d$, equipped with the TV norm between $P,Q \in \cP(\R^d)$ defined by $\|P - Q\|_\tv \defeq \frac{1}{2}|P - Q|(\cZ)$. We say that $Q$ is an $\eps$-deletion of $P$ if $Q \leq \frac{1}{1-\eps} P$ (where such inequalities are set-wise). We write $\E_P[f(X)]$ for expectation of $f(X)$ with $X\sim P$; when clear from the context, the random variable is dropped and we write $\E_P[f]$. Let $\mu_P$ denote the mean and $\Sigma_P$ the covariance matrix of $P \in \cP(\R^d)$, and let $\cP_p(\R^d) \defeq \{ P \in \cP(\R^d) : \E_P[\|X - \mu_P\|^p] < \infty \}$. %
The push-forward of $f$ through $P\in\cP(\R^d)$ is $f_\sharp  P(\cdot) \defeq P(f^{-1}(\cdot))$. The set of positive integers up to $n \in \mathbb N$ is denoted by $[n]$; we also use the shorthand $[x]_+=\max\{x,0\}$. We write $\lesssim, \gtrsim, \asymp$ for inequalities/equality up to absolute constants, and let $a \lor b \defeq \max\{a,b\}$. For a matrix $A \in \R^{d \times d}$, we write $\|A\|_\op \defeq \sup_{x \in \unitsph} \|Ax\|$ for its operator norm. If $A$ is further is diagonalizable, we write $\lambda_{\max}(A) = \lambda_1(A) \geq \dots \geq \lambda_d(A)$ for its eigenvalues.

\paragraph{Classical and outlier-robust Wasserstein distances.} For $p \in [1,\infty)$, the \emph{$p$-Wasserstein distance} between $P,Q\in \cP_p(\R^d)$ is 
$\Wp(P,Q) \defeq \inf_{\pi \in \Pi(P,Q)} \left(\EE_\pi\big[\|X-Y\|^p\big]\right)^{1/p}$, where $\Pi(P,Q)\defeq\{\pi\in\cP(\cZ^2):\,\pi(\cdot\times\cZ)=P,\,\pi(\cZ\times\cdot)=Q\}$ is the set of all their couplings. Some basic properties of $\Wp$ are (see, e.g., \cite{villani2003,santambrogio2015}): (i) $\Wp$ is a metric on $\cP_p(\cZ)$; (ii) the distance is monotone in the order, i.e., $\Wp\leq \mathsf{W}_q$ for $p\leq q$; and (iii) $\Wp$ metrizes weak convergence plus convergence of $p$th moments: $\Wp(P_{n},P) \to 0$ if and only if $P_{n} \stackrel{w}{\to} P$ and $\int \|x\|^{p} d P_{n}(x) \to \int \|x\|^{p} d P(x)$.
For a family of measures $\cG \subseteq \cP(\R^d)$, we write $\Wp(P,\cG) \defeq \inf_{R \in \cG} \Wp(P,R)$.\medskip

To handle corrupted data, we employ the \emph{$\eps$-outlier-robust $p$-Wasserstein distance\footnote{While not a metric, $\RWp$ is symmetric and satisfies an approximate triangle inequality (\cite{nietert2023robust}, Proposition 3).}}, defined by
\begin{equation}
\vspace{-0.5mm}
    \RWp(\mu,\nu) \defeq \inf_{\substack{\mu' \in \cP(\R^d)\\\|\mu' - \mu\|_\tv \leq \eps}} \Wp(\mu',\nu) = \inf_{\substack{\nu' \in \cP(\R^d)\\\|\nu' - \nu\|_\tv \leq \eps}} \Wp(\mu,\nu').\label{eq:RWp}
\vspace{-0.5mm}
\end{equation}
This is an instance of partial OT, and the second equality is a useful consequence of Lemma 4 in \cite{nietert2023robust} (see Appendix A of \cite{nietert2023outlier} for further details).

\section{Robust Distribution Learning}
\label{sec:distribution-learning}

We now turn to robust distribution estimation under combined TV-$\Wone$ contamination. Given corrupted samples from an unknown distribution $P$, we aim to produce an estimate $\hat{P}$ such that $\Wonek(\hat{P},P)$ is appropriately small for all $k \in [d]$.
When $k=d$, we recover standard $\Wone$. When $k=1$, we shall see that the resulting estimation task is of essentially the same complexity as mean estimation. Omitted proofs appear in \cref{app:distribution-learning-population-limit,app:distribution-learning-finite-sample}.

\subsection{The Population Limit}
\label{ssec:distribution-learning-population-limit}

We first examine the information-theoretic limits of this problem without sampling error, namely, allowing computationally-intractable estimators and access to population distributions (rather than samples). We consider learning under the following environment. %

\begin{tcolorbox}[colback=white]
\label{setting:A}
\textbf{Setting A:} Fix corruption levels $0 \leq \eps \leq 0.49$\footnotemark{} and $\rho \geq 0$, along with a clean distribution family $\cG \subseteq \cP(\R^d)$. Nature selects $P \in \cG$, and the learner observes $\tilde{P}$ such that $\RWone(\tilde{P},P) \leq \rho$.
\end{tcolorbox}
\footnotetext{As $\eps$ approaches the optimal breakdown point of $1/2$, it becomes information-theoretically impossible to distinguish inliers from outliers. The quantity $0.49$ can be replaced with any constant bounded away from $1/2$.}

Given $\tilde{P}$, we seek an estimator  $\hat{P}$ such that $\Wpk(\hat{P},P)$ is small for all $k$. To ensure that effective learning is possible, we impose a stability condition on the clean measure $P$.

\begin{definition}[Stability]
Let $0 < \eps < 1$ and $\delta \geq \eps$. We say that a distribution $P \in \cP(\R^d)$ is $(\eps,\delta)$-stable if, for all $Q \leq \frac{1}{1-\eps}P$, we have $\|\mu_Q - \mu_P\| \leq \delta$ and $\|\Sigma_Q - \Sigma_P\|_\op \leq \delta^2/\eps$. Write $\cS(\eps,\delta)$ for the family of $(\eps,\delta)$-stable $P$ such that $\Sigma_P \preceq I_d$, and $\Siso(\eps,\delta)$ for the subfamily of those for which $\Sigma_P \succeq (1 - \delta^2/\eps)I_d$.
\end{definition}

A distribution is stable if its first two moments vary minimally under $\eps$-deletions. 
The near-isotropic subfamily $\Siso$ coincides with a popular definition in algorithmic robust statistics (see, e.g., Chapter 2 of \citealp{diakonikolas2023algorithmic}). Without the bound on $\|\Sigma_Q - \Sigma_P\|_\op$, this definition coincides with that of resilience, a standard sufficient condition for (inefficient) robust mean estimation \citep{steinhardt2018resilience}. Stability is a flexible notion that connects to many standard tail bounds.

\begin{example}[Concrete stability bounds]
\label{ex:stability-bounds}
\prf{stability-bounds}
Fix $0 < \eps \leq 0.99$\footnote{Similarly to the $\eps \leq 0.49$ bound above, here $0.99$ can be replaced with any constant bounded away from 1.} and $P \in \cP(\R^d)$ with $\Sigma_P \preceq I_d$. Then:
\begin{itemize}
    \itemsep 0em 
    \item \textbf{Bounded covariance}: with no further assumptions, $P \in \Siso(\eps,O(\sqrt{\eps}))$;
    \item \textbf{Sub-Gaussian}: if $P$ is $1$-sub-Gaussian, then $P \in \cS(\eps,O(\eps \sqrt{\log(1/\eps)}))$;
    \item \textbf{Log-concave}: if $P$ is log-concave, then $P \in \cS(\eps,O(\eps \log(1/\eps)))$;
    \item \textbf{Bounded moments of order $\bm{q\geq 2}$}: if $\sup_{v \in \unitsph}\E_P[|v^\top(Z - \mu_P)|^q] \leq 1$, then $P \in \cS(\eps,O(\eps^{1 - 1/q}))$.
\end{itemize}
We refer to the families of distributions satisfying these properties by $\Gcov$, $\cG_{\mathrm{subG}}$, $\cG_\mathrm{lc}$, and $\cG_q$, respectively. Note that $\Gcov = \cG_2$. Similar bounds are derived in Chapter 2 of \cite{diakonikolas2023algorithmic}.
\end{example}

We now present our primary risk bound for the population-limit.

\begin{theorem}[Population-limit risk bound]
\label{thm:pop-limit-risk-bd}
Under Setting \hyperref[setting:A]{A}, take $\eta = \min \{ 2\eps, 1/4 + \eps/2 \}$ and assume that $\cG \subseteq \cS(2\eta,\delta)$. Then the minimum distance estimate\footnote{Here and throughout, the existence of such minimizers is not consequential but simply assumed for cleaner statements. Approximate minimizers up to some additive error provide the same risk bounds up to said error.} $\hat{P}_\mathrm{MDE} = \argmin_{Q \leq \frac{1}{1-\eta}}\Wtwo(Q,\cG)$ satisfies
\vspace{-2ex}
\begin{align*}
    \Wonek(\hat{P}_\mathrm{MDE},P) \lesssim \sqrt{k}\delta + \rho, \quad \forall k \in [d].
\end{align*}
\end{theorem}
The minimum distance estimate $\hat P_{\text{MDE}}$ involves an infinite dimensional optimization problem, which is computationally intractable. In the subsequent subsection we propose an iterative filtering algorithm that approximately solves a surrogate optimization problem on a finite sample set.\medskip
\begin{proof}\ \ 
The constant $\eta$ is selected so that $\eta - \eps \gtrsim \eps$ while keeping $2\eta \leq 0.99$ bounded away from 1. For concreteness, the reader may want to focus on the case where $\eps \leq 1/6$ and $\eta = 2\eps$. \medskip

To treat combined Wasserstein and TV contamination, we first show that any $\Wone$ perturbation can be decomposed into a $\Wtwo$ perturbation followed by a TV perturbation.

\begin{lemma}[$\bm{\Wone}$ decomposition]
\label[lemma]{lem:W1-decomposition}
\prf{W1-decomposition}
Fix $0 < \tau < 1$ and $P,Q \in \cP(\R^d)$ with $\Wone(P,Q) \leq \rho$. Then there exists $R \in \cP(\R^d)$ such that $\Wone(P,R) \leq \rho$, $\Wtwo(P,R) \leq \sqrt{2}\rho/\sqrt{\tau}$, and $\|R - Q\|_\tv \leq \tau$.
\end{lemma}

The proof in \cref{prf:W1-decomposition} takes $Z \sim P$ and $Z + \Delta \sim Q$, where $\E[\|\Delta\|] \leq \rho$. Letting $E$ be the event such that $\|\Delta\|$ is less than its $1-\tau$ quantile, we conclude by setting $R$ to the law of $Z + \Delta \mathds{1}_E$.\medskip

Next, we prove a version of the theorem when $\rho = 0$, but the clean measure is close to $\cS(\eps,\delta)$ under $\Wtwo$.

\begin{lemma}[Risk bound when $\bm{\rho = 0}$]
\label[lemma]{lem:TV-risk-bound}
\prf{TV-risk-bound}
Fix $0 \leq \eta < 1/2$, $\lambda \geq 0$, and $\cG \subseteq \cS(2\eta,\delta)$. Take $R,\tilde{R} \in \cP(\R^d)$ such that $\Wtwo(R,\cG) \leq \lambda$ and $\|R - \tilde{R}\|_\tv \leq \eta$. Then the estimate $\smash{\hat{R} = \argmin_{Q \leq \frac{1}{1-\eta} \tilde{R}} \Wtwo(Q,\cG)}$ satisfies $\smash{\Wonek(\hat{R},R) \lesssim \frac{1}{1-2\eta}( \sqrt{k}\delta + \lambda\sqrt{\eta})}$, for all $k \in [d]$.
\end{lemma}

The proof in \cref{prf:TV-risk-bound} observes that $R$ satisfies a generalized stability bound: for all $R' \leq \frac{1}{1-\eps}R$ and $M \succeq 0$, we have $\|\mu_{R'} - \mu_R\| \lesssim \delta + \lambda \sqrt{\eps}$ and 
\begin{equation}
\label{eq:generalized-stability}
    \bigl|\tr\big(M(\Sigma_{R'} - \Sigma_R)\big)\bigr| \lesssim \frac{\delta^2}{\eps}\tr(M) + \lambda^2 \|M\|_\op.
\end{equation}
In words, the latter bounds shows that the gap $\Sigma_{R'} - \Sigma_R$ lies in the Minkowski sum of an operator norm ball of radius $O(\delta^2/\eps)$ and a trace norm ball of radius $O(\lambda^2)$. In contrast, the more direct guarantee $\|\Sigma_{R'} - \Sigma_R\|_\op \lesssim \delta^2/\eps + \lambda^2$ would only give a suboptimal risk bound of $\sqrt{k} \delta + \sqrt{k}\lambda$.
\medskip

Given the above lemmas, we are ready to prove the theorem. Fix $P \in \cG$ and $\tilde{P}$ such that $\RWone(\tilde{P},P) \leq \rho$. This requires the existence of $Q$ such that $\Wone(P,Q) \leq \rho$ and $|Q - \tilde{P}|_\text{TV} \leq \epsilon$. Applying \cref{lem:W1-decomposition} to $P$ and $Q$ with $\tau = \eta - \eps$ implies that there exists $R$ such that $\Wone(P,R) \leq \rho$, $\Wtwo(P,R) \leq \sqrt{2}\rho/\sqrt{\eta - \eps} \eqqcolon \lambda$, and $\|R - \tilde{P}\|_\tv \leq \|R - Q\|_\tv + \|Q - \tilde{P}\|_\tv \leq \eta - \eps + \eps = \eta$. Applying \cref{lem:TV-risk-bound} with TV corruption level $\eta$ and $\Wtwo$ bound $\lambda \lesssim \rho/\sqrt{\eps}$, we find that $\hat{P}_\mathrm{MDE}$ from the theorem statement satisfies
\begin{equation*}
    \Wonek(\hat{P}_\mathrm{MDE},P) \leq \Wonek(\hat{P}_\mathrm{MDE},R) + \Wonek(R,P) \lesssim \frac{1}{1-2\eta}\bigl(\sqrt{k}\delta + \lambda{\sqrt{\eta}}\bigr) \lesssim \frac{1}{1-2\eta}\bigl(\sqrt{k}\delta + \rho\bigr),
\end{equation*}
as desired.
\end{proof}

This bound is tight for many distribution families, including those in \cref{ex:stability-bounds}.

\begin{corollary}
\label[corollary]{cor:pop-limit-risk-bounds}
\prf{pop-limit-risk-bounds}
The minimum distance estimate $\hat{P}_\mathrm{MDE}$ from \cref{thm:pop-limit-risk-bd} achieves error
\begin{align*}
   \Wonek(\hat{P}_\mathrm{MDE},P) \lesssim \begin{cases}
        \sqrt{k\eps} + \rho, & \cG = \Gcov\\
        \sqrt{k}\eps\sqrt{\log(1/\eps)} + \rho, & \cG = \cG_{\mathrm{subG}}\\
        \sqrt{k}\eps\log(1/\eps) + \rho, &\cG = \cG_{\mathrm{lc}}\\
        \sqrt{k}\eps^{1 - 1/q} + \rho, &\cG = \cG_{q}
    \end{cases},
\end{align*}
and each of these guarantees is minimax optimal up to logarithmic factors in $\eps^{-1}$.
\end{corollary}

For the minimax lower bounds, we employ existing constructions for the setting where $\rho = 0$. To strengthen these bounds when $\rho > 0$, we show that the learner cannot distinguish between translations of magnitude $\rho$.

\begin{remark}[Comparison to other minimum distance estimators]
Estimators related to that in \cref{thm:pop-limit-risk-bd} are standard in robust statistics (see, e.g., \cite{donoho88,zhu2019resilience} for methods based on (smoothed) TV projection) and robust optimal transport (see, e.g., \cite{nietert2023robust}, which employs projection under $\RWp$). The risk bounds from \cref{lem:TV-risk-bound} match those in the literature for robust mean and distribution estimation when $\rho = 0$ (recalling that our results extend to mean estimation since $\|\mu_P - \mu_Q\| \leq \mathsf{W}_{1,1}(P,Q)$). We diverge from these existing estimators by returning $\hat{P} = \mathsf{T}(\tilde{P})$ which lies not in $\cG$ but nearby $\cG$ under $\Wtwo$. The fact that $\hat{P}$ is an $\eps$-deletion of $\tilde{P}$ is essential in turning this approach into a practical algorithm in \cref{ssec:distribution-learning-finite-sample}.
\end{remark}

\subsection{Finite-Sample Algorithms}
\label{ssec:distribution-learning-finite-sample}

We now turn to the finite-sample setting. Here, our rates are only tight when $\delta \gtrsim \sqrt{\eps}$, so we restrict to the family $\Gcov$ of distributions $P \in \cP(\R^d)$ with $\Sigma_P \preceq I_d$. Indeed, $\Gcov \subseteq \cS(\eps,O(\sqrt{\eps}))$ by \cref{ex:stability-bounds}.

\begin{tcolorbox}[colback=white]
\label{setting:B}
\textbf{Setting B:}
Let $0<\eps<\eps_0$, where $\eps_0$ is a sufficiently small absolute constant\footnotemark{}. Fix $\rho \geq 0$ and sample size $n = \Omega(d \log (d)/\eps)$. Nature samples $X_1, \dots, X_n$ i.i.d.\ from $P \in \Gcov$, with empirical measure $P_n$. The learner observes $\tilde{X}_1, \dots \tilde{X}_n$ with empirical measure $\tilde{P}_n$ such that $\RWp(\tilde{P}_n,P_n) \leq \rho$.
\end{tcolorbox}
\footnotetext{We make no effort to optimize the breakdown point $\eps_0$. Similar results for robust mean estimation first required $\eps_0 \ll 1/2$, but this was later alleviated \citep{hopkins2020robust,zhu2022robust,dalalyan2022all}. We expect that similar improvements are possible under our setting but defer such optimization future work---see \cref{sec:summary} for additional discussion.}

We aim to match the bound of \cref{thm:pop-limit-risk-bd}, computing an estimate $\hat{P}_n$ such that $\Wonek(\hat{P}_n,P) \lesssim \sqrt{k\eps} + \rho$ for sufficiently large $n$.
In order to turn the $\Wtwo$ projection procedure into an efficient algorithm, we 
replace the intractable objective $\Wtwo(Q,\Gcov)$ with the tractable trace norm objective $\tr(\Sigma_Q - I_d)_+ = \sum_i [\lambda_i(\Sigma_Q) - 1]_+$, which can be computed via eigen-decomposition. 

\begin{lemma}[Trace norm comparison]
\label[lemma]{lem:W2-trace-norm-comparison}
\prf{W2-trace-norm-comparison}
For $Q \in \cP(\R^d)$, we have
\begin{equation*}  %
\tfrac{1}{2}\tr(\Sigma_Q -2 I_d)_+  \leq \Wtwo\bigl(Q,\Gcov\bigr)^2\leq \tr(\Sigma_Q -I_d)_+.
\end{equation*}
\end{lemma}

This result underlies \FilterSimple{} (\cref{alg:W2-project}), which approximately solves the optimization problem $\min_{Q \leq \frac{1}{1-O(\eps)}\tilde{P}_n} \tr(\Sigma_Q - \sigma^2I_d)_+$ using a variant of iterative filtering \citep{diakonikolas2016robust}. In the algorithm description, we identify a multiset $T \subseteq \R^d$ with the uniform distribution $\Unif(T)$. We emphasize that the high-level idea of trimming samples from a corrupted observation to control the empirical covariance matrix is a familiar paradigm in algorithmic robust statistics (see, e.g., \citealp{klivans2009learning,diakonikolas2016robust,diakonikolas2018robustly,steinhardt2018resilience}). Our main contributions are showing that this approach still applies with local adversarial corruptions and identifying $\tr(\Sigma_Q - I_d)_+$ as the appropriate quantitative measure for covariance magnitude.

\begin{theorem}
\label{thm:W2-project}
\prf{W2-project}
Under Setting \hyperref[setting:B]{B}, $\FilterSimple(\tilde{P}_n,\eps,\rho)$ returns $\hat{P}_n$ in time $\poly(n,d)$ such that
\begin{equation*}
    \Wonek(\hat{P}_n,P) \lesssim \sqrt{k\eps} + \rho + \Wonek(P,P_n), \qquad \forall k \in [d].
\end{equation*}
with probability at least $2/3$.
\end{theorem}

\begin{algorithm2e}[t]
\caption{\FilterSimple}
\label{alg:W2-project}
\LinesNumbered
\DontPrintSemicolon
\KwIn{Contamination levels $\eps$ and $\rho$, uniform discrete measure $\tilde{P}_n$ supported on $T \subseteq \R^d$}
\KwOut{Uniform discrete measure $\hat{P}_n$}
$\sigma \gets 50, C \gets 10^{10}$\;
Compute eigen-decomposition $\lambda_1, \dots, \lambda_d \in \R$, $v_1, \dots, v_d \in \R^d$ of $\Sigma_T - \sigma^2I_d$\;
$\Pi \gets \sum_{i:\lambda_i \geq 0} v_i v_i^\top$\;
\lIf(\tcp*[f]{LHS equals $\tr(\Sigma_T - \sigma^2I_d)_+$}){$\tr(\Pi(\Sigma_T - \sigma^2I_d)) < C\eps + C\rho^2/\eps$}{\Return $\hat{P}_n = \Unif(T)$}\label{step:W2-project-terminate}
\Else{
    $g(x) \gets \|\Pi(x - \mu_T)\|^2$ for $x \in T$\label{step:W2-project-g-def}\;
    Let $L \subseteq T$ be set of $6\eps|T|$ points for which $g(x)$ is largest\label{step:W2-project-L-def}\;
    $f(x) \gets g(x)$ for $x \in L$ and $f(x) \gets 0$ otherwise\label{step:W2-project-f-def}\;
    Remove each point $x \in T$ from $T$ with probability $f(x)/\max_{x \in T}f(x)$\;
    Return to Step 1 with new set $T$\;
}
\end{algorithm2e}

Over $P \in \Gcov \subseteq \cS(\eps,O(\sqrt{\eps}))$, this guarantee attains the minimax optimal error from \cref{cor:pop-limit-risk-bounds} as the sample size $n$ increases (whence the empirical estimation error vanishes). Our proof shows that the estimate $\hat{P}_n$ satisfies $\tr(\Sigma_{\smash{\hat{P}_n}} - O(1)I_d)_+ \lesssim \rho^2/\eps$, mirroring the martingale-based analysis of iterative filtering with the simpler objective $\lambda_{\max}(\Sigma_Q)$; see, e.g., Section 2.4 of \cite{diakonikolas2023algorithmic}. 
Via \cref{lem:W2-trace-norm-comparison}, we then convert this trace norm bound into a $\Wtwo$ bound, and proceed with the analysis of the $\Wtwo$ projection from \cref{thm:pop-limit-risk-bd} to arrive at the risk bound above. As with \cref{thm:pop-limit-risk-bd}, the generalized stability bound \eqref{eq:generalized-stability} is essential for avoiding a $\sqrt{k}\rho$ dependence.\medskip

The remaining empirical convergence term, $\Wonek(P,P_n)$, can always be bounded by $\Wone(P,P_n)$, and the covariance bound implies that $\E[\Wone(P,P_n)] = \tilde{O}(\sqrt{d}n^{-1/d})$ for $d \geq 2$ (see, e.g., Theorem 3.1 of \citealp{lei2020convergence}). Generally, we would hope for a faster $n^{-1/k}$ rate, and this is indeed the case under appropriate additional assumptions on the clean distribution $P$. To name a few instances, \citealp{lin2021projection} derive such rates for general $k$ under a Bernstein tail condition or a Poincar\'e inequality assumption, while \citealp{niles2022estimation} provide rates when $P$ satisfies a transport inequality \citep{niles2022estimation} (which, in particular, holds for sub-Gaussian distributions). Empirical convergence rates in additional settings have been derived in the $k=1$ case, e.g., for log-concave distributions \citep{nietert2022statistical} and under certain isotropic and moment boundedness assumptions \citep{bartl2022structure}. Recently, \citet{boedihardjo2024sharp} provided bounds for general $k$ which apply under our covariance bound alone. We combine these here
with the bound from \cref{thm:W2-project}.

\begin{corollary}[Statistical performance]
\label[corollary]{cor:W2-project-statistical}
\prf{W2-project-statistical}
Under Setting \hyperref[setting:B]{B},
$\FilterSimple(\tilde{P}_n,\eps,\rho)$ returns $\hat{P}_n$ in time $\poly(n,d)$ such that
\begin{equation}
\label{eq:statistical-upper-bound}
    \Wonek(\hat{P}_n,P) \lesssim \sqrt{k\eps} + \rho + \tilde{O}\left(k\sqrt{d}n^{-\frac{1}{k \lor 2}}\right).
\end{equation}
with probability at least $2/3$.
\end{corollary}

With respect to optimality, we note that the first two terms in \eqref{eq:statistical-upper-bound} are necessary, due to the minimax lower bound for $\Gcov$ within \cref{cor:pop-limit-risk-bounds}. Further, even when there are no corruptions and $\rho = \eps = 0$, a minimax lower bound of \cite{niles2022estimation} implies that $\Omega(c_d n^{-1/(k \lor 2)} + \sqrt{d/n})$ error is unavoidable, even for $P$ supported on $[0,1]^d$. We defer a tight characterization of finite-sample risk for future work.

\begin{remark}[Recovering standard filtering via sliced $\bm{\Wtwo}$ projection]
We note that \cref{lem:W2-trace-norm-comparison} can be adapted to the sliced $\Wtwo$ setting. In particular, one can approximate $\mathsf{W}_{2,1}(Q,\Gcov)$ by the operator norm $\|(\Sigma_Q - I_d)_+\|_\op = [\|\Sigma_Q\|_\op - 1]_+$. This is equivalent to the standard objective $\|\Sigma_Q\|_\op$ for iterative filtering \citep{diakonikolas2016robust} (when $\|\Sigma_Q\|_\op > 1$, and otherwise the algorithm will have terminated), providing a new perspective on this standard algorithm.
\end{remark}

\begin{remark}[Comparison to \citealp{diakonikolas2018robustly}]
\label{remark:comp-to-gaussian-mean-estimation}
Multi-directional filtering of the empirical covariance matrix, as employed by \FilterSimple{}, is also used by \citet{diakonikolas2018robustly} for robust Gaussian mean estimation. Their Algorithm 4 identifies a relatively small subspace $V \subseteq \R^d$, whose dimension is constrained to not exceed $O(\log(1/\eps))$, along which $\Sigma_T - I_d$ has large eigenvalues. Then, it employs a brute-force approach to obtain a robust estimate $\tilde{\mu}$ for $\Pi_V(\mu_P)$ and filters out points $x \in T$ for which $\|\Pi_V x \!-\! \tilde{\mu}\|^2$ is large. In contrast, we filter based upon $\|\Pi (x-\mu_T)\|^2$, where $\mu_T$ is simply the empirical mean and $\Pi$ projects onto the subspace spanned by all eigenvectors of $\Sigma_T\!-\!I_d$ with sufficiently large eigenvalues. Further, their stopping condition is based on the count of large eigenvalues rather than our trace norm~objective. \end{remark}

\subsection{Other Corruption Models and Robust Mean Estimation}
\label{ssec:istribution-learning-extensions}

We now discuss some complementary results. First, we remark that if $\eps = 0$, then $\tilde{P}_n$ itself satisfies the bound $\Wonek(\tilde{P}_n,P) \leq \rho + \Wonek(P_n,P)$, trivially matching the previous upper bounds. In the case that $\rho = 0$ and we only suffer TV corruption, standard iterative filtering resolves the question of efficient distribution learning for near-isotropic $P$.

\begin{proposition}
\label[proposition]{prop:rho-zero}
\prf{rho-zero}
Under Setting \hyperref[setting:B]{B} with $\rho = 0$ and $P \in \Siso(4\eps,\delta)$, any estimate $\hat{P} \leq \frac{1}{1-4\eps}\tilde{P}_n$ such that $\|\Sigma_{\hat{P}}\|_\mathrm{op} \leq 1 + O(\delta^2/\eps)$ satisfies $\Wonek(\hat{P},P) \lesssim \sqrt{k} \delta + \Wonek(P,P_n)$, for all $k \in [d]$.
\end{proposition}

Indeed, this $\lambda_{\max}$ bound is achieved by all stability-based algorithms for robust mean estimation (see, e.g., Theorem 2.11 of \citealp{diakonikolas2023algorithmic}).
Our proof employs a refined version of the certificate lemma for stable distributions (see Lemma 2.7 of \citealp{diakonikolas2023algorithmic}). The isotropic restriction is standard in algorithmic robust statistics; hardness results suggest it cannot be eliminated without imposing further assumptions like Gaussianity or losing computational tractability \citep{hopkins19hard}.\medskip

Next, we comment on the simpler task of robust mean estimation. Under Setting~\hyperref[setting:B]{B}, we can simply return the mean of $\hat{P}_n = \FilterSimple(\tilde{P}_n,\eps,\rho)$ to obtain error $O(\sqrt{\eps} + \rho)$. It is not hard to show that standard iterative filtering also suffices, employing the $\Wone$ decomposition in \cref{lem:W1-decomposition} to bound the extent to which the Wasserstein corruption can perturb second moments after filtering out its $\eps$-tails. However, neither approach generalizes to $\Siso(\eps,\delta)$, leaving us with a simple open question: \emph{under Setting~\hyperref[setting:B]{B} with $P = \cN(\mu,I_d)$ and $n = \poly(d,1/\eps)$, can one efficiently compute $\hat{\mu}$ from $\tilde{P}_n$ such that $\|\hat{\mu} - \mu\| = \tilde{O}(\eps + \rho)$?}\medskip

\section{Robust Stochastic Optimization}
\label{sec:robust:stochastic}
Finally, we present an application to robust stochastic optimization. We consider a setting where the learner seeks to make a decision $\hat{\theta} \in \Theta$ that performs well on a data distribution $P$, given only a corrupted observation $\tilde{P}_n$. More precisely, given a loss function $L : \Theta \times \R^d \to \R$, we seek to minimize the risk $\E_P[L(\hat{\theta},X)]$. In the following we suppress dependence of $L$ on the model parameters $\theta\in\Theta$ and write $\ell(\cdot) = L(\theta,\cdot)$ for a specific function. We also introduce the set $\cL=\{L(\theta,\cdot)\}_{\theta\in\Theta}$ for the whole class, and impose the following.

\begin{assumption}
\label{assumption:DRO}
Fix $p \geq 1$. Take $\cL$ to be a family of real-valued loss functions on $\R^d$, such that each $\ell \in \cL$ is of the form $\ell = \underline{\ell} \circ A$, where $A:\R^d \to \R^k$ is affine and $\underline{\ell} : \R^k \to \R^d$ is l.s.c.\ with $\sup_{z \in \cZ} \frac{\ell(z)}{1 + \|z\|^p} < \infty$.
\end{assumption}

In addition to mild regularity conditions, we assume that the loss functions operate on $k$-dimensional linear features of the data. For example, this captures $k$-variate linear regression if $\Theta \subseteq \R^{d \times k}$ and $L$ maps $\theta \in \Theta$ and $(x,y) \in \R^{d-k} \times \R^k$ to $L(\theta,(x,y)) = \|\theta x - y\|^p$. In the worst case, we may always set $k=d$.\medskip

If it is known that $\tilde{P}_n$ and $P$ are close under $\Wp$, a popular decision-making procedure is \emph{Wasserstein distributionally robust optimization (WDRO)}, which selects
\begin{equation*}
    \hat{\ell}_\mathrm{WDRO} \defeq \argmin_{\ell \in \cL} \sup_{Q:\,\Wp(Q,\tilde{P}_n) \leq r} \E_Q[\ell].
\end{equation*}
Indeed, if $\Wp(P,\tilde{P}_n) \leq r$, then it is easy to prove the excess risk bound
\begin{equation}
\label{eq:WDRO-excess-risk}
    \E_P[\hat{\ell}_\mathrm{WDRO}] - \E_P[\ell_\star] \lesssim \sup_{\Wp(Q,P) \leq 2r} \E_Q[\ell_\star] - \E_P[\ell_\star] \eqqcolon \cR_p(\ell_\star;P,2r),
\end{equation}
where $\ell_\star = \argmin_{\ell \in \cL} \E_P[\ell]$. The right-hand side, denoted $\cR_p$, is termed the \emph{$p$-Wasserstein regularizer} and characterizes a certain variational complexity of the optimal loss function (see, e.g., \citealp{gao2023distributionally}). In particular, we have $\cR_1(\ell;P,r) \leq r\|\ell\|_{\Lip}$.\medskip

Alas, the assumption that $\Wp(P,\tilde{P}_n)$ is small is quite conservative, especially in the high-dimensional setting, where Wasserstein empirical convergence rates suffer from the curse of dimensionality. In fact, given the low-dimensional structure imposed in Assumption \ref{assumption:DRO}, it is natural to expect that a much smaller Wasserstein radius would suffice, e.g., as captured by the $k$-dimensional sliced distance. The next theorem indeed shows that the inner WDRO maximization problem automatically adapts to the dimensionality of a given loss function, which provide a new perspective on beating the curse of dimensionality in WDRO, as discussed in detail at the end of this section.

\begin{theorem}
\label{thm:WDRO}
Fix $P,\hat{P} \in \cP_p(\R^d)$ with $\Wpk(P,\hat{P}) \leq \tau$, for some $\tau \geq 0$. Under Assumption~\ref{assumption:DRO}, we have
\begin{equation*}
    \E_P[\ell] \leq \sup_{Q \in \cP(\R^d) : \,\Wp(\hat{P},Q) \leq \tau } \E_Q[\ell]
\end{equation*}
for each $\ell \in \cL$.
\end{theorem}

\begin{proof}\ \ 
Take $\ell = \underline{\ell} \circ A$ to be the decomposition guaranteed by Assumption~\ref{assumption:DRO}. Assume without loss of generality that $A$ is linear. By the QR decomposition, we can rewrite $\ell$ as $\underline{\ell} \circ BU$ for $U \in \R^{k \times d}$ such that $UU^\top = I_k$. Take $\tilde{\ell} = \underline{\ell} \circ B$. We now show that
\begin{equation}
\label{eq:WDRO-helper}
     \sup_{Q \in \cP(\R^d) :\, \Wp(\hat{P},Q) \leq \tau } \E_Q[\ell] = \sup_{R \in \cP(\R^k) :\, \Wp(U_\sharp\hat{P},R) \leq \tau } \E_R[\tilde{\ell}].
\end{equation}
For the ``$\leq$'' direction, note that for any feasible $Q$ for the left supremum, $R = U_\sharp Q$ is feasible for the right hand side with equal objective value. For the ``$\geq$'' direction, take any $R$ feasible for the right supremum. Let $(UX,Y)$ be an optimal coupling for the $\Wp(U_\sharp\hat{P},R)$ problem, where $X \sim \hat{P}$. Taking $Q$ to be the law of $X + U^\top (Y-UX)$, we have that $\Wp(\hat{P},Q)^p \leq \E[\|U^\top (Y-UX)\|^p] = \E[\|Y - UX\|^p] = \Wp(U_\sharp\hat{P},R)^p \leq \tau$, and $\E_Q[\ell] = \E[\tilde{\ell}(U X + U U^\top (Y - UX))] = \E[\tilde{\ell}(Y)] = \E_R[\tilde{\ell}]$, as desired.\medskip

At this point, we note that $\Wp(U_\sharp P,U_\sharp \hat{P}) \leq \Wpk(P,\hat{P}) \leq \tau$ and bound
\begin{align*}
    \E_P[\hat{\ell}] &= \E_{U_\sharp P}[\tilde{\ell}]\\
    &\leq \sup_{R \in \cP(\R^k): \,\Wone(R,U_\sharp\hat{P}) \leq \tau} \E_R[\tilde{\ell}]\\
    &= \sup_{Q \in \cP(\R^d): \,\Wone(Q,\hat{P}) \leq \tau} \E_Q[\hat{\ell}] - \E_P[\ell_\star], \tag*{\eqref{eq:WDRO-helper}}
\end{align*}
as desired.
\end{proof}

\cref{thm:WDRO} implies that we may center the WDRO procedure around any distribution $\hat P$, for which we have control over its $\Wonek$ distance from the true population $P$. Remarkably, the $\FilterSimple$ algorithm provides a computationally efficient way to find such a distribution, and \cref{thm:W2-project} further yields the required bound on the $\Wonek$ error. We have the following.

\begin{corollary}
\label{cor:OR-WDRO}
Under Setting \hyperref[setting:B]{B} and Assumption~\ref{assumption:DRO} with $p=1$, take $\hat{P}_n = \FilterSimple(\tilde{P}_n,\eps,\rho)$ and let $\tau$ be any upper bound on the error $\Wonek(\hat{P}_n,P) \lesssim \sqrt{k\eps} + \rho + \Wonek(P,P_n)$. Then the WDRO estimate
\begin{equation*}
    \hat{\ell} = \argmin_{\ell \in \cL} \sup_{Q \in \cP(\R^d):\,\Wone(\hat{P}_n,Q) \leq \tau} \E_Q[\ell]
\end{equation*}
satisfies the excess risk bound $\E_P[\ell] - \E[\ell_\star] \leq 2\|\ell_\star\|_{\Lip} \tau$, where $\ell_\star = \argmin_{\ell \in \cL}\E_P[\ell]$.
\end{corollary}
\begin{proof}
As in \cref{thm:WDRO}, decompose $\ell_\star = \tilde{\ell} \circ U$ for $U \in \R^{k \times d}$ such that $UU^\top = I_k$. Note that $\|\tilde{\ell}\|_{\Lip} = \|\ell_\star\|_{\Lip}$ and that $\Wone(U_\sharp P,U_\sharp \hat{P}_n) \leq \Wonek(P,\hat{P}_n) \leq \tau$. 
We then bound
\begin{align*}
    \E_P[\hat{\ell}] - \E_P[\ell_\star] &\leq \sup_{Q \in \cP(\R^d):\, \Wone(Q,\hat{P}_n) \leq \tau} \E_Q[\hat{\ell}] - \E_P[\ell_\star] \tag{\cref{thm:WDRO}}\\
    &\leq \sup_{Q \in \cP(\R^d): \,\Wone(Q,\hat{P}_n) \leq \tau} \E_Q[\ell_\star] - \E_P[\ell_\star]\tag{$\hat{\ell}$ minimizing}\\
    &\leq \sup_{R \in \cP(\R^k): \Wone(R,U_\sharp\hat{P}_n) \leq \tau} \E_R[\tilde{\ell}] - \E_{U_\sharp P}[\tilde{\ell}]\tag{$R = U_\sharp Q$}\\
    &\leq \sup_{R \in \cP(\R^k): \,\Wone(R,U_\sharp P) \leq 2\tau} \E_R[\tilde{\ell}] - \E_{U_\sharp P}[\tilde{\ell}] \tag{$\Wone(U_\sharp P, U_\sharp \hat{P}_n) \leq \tau$}\\
    &\leq \|\tilde{\ell}\|_{\Lip} 2\tau\\
    &= 2\|\ell_\star\|_{\Lip} \tau,
\end{align*}
as desired.
\end{proof}

\subsection{Beating the Curse of Dimensionality in WDRO}

Despite promising applications, the classic Wasserstein DRO approach suffers from the curse of dimensionality. The rate of empirical convergence under the Wasserstein distance scales as $n^{-1/d}$, which cannot be generally improved when $d\geq 3$ \citep{fournier2015rate,lei2020convergence}. In light of this rate, \citet{mohajerin2018data} showed that if the WDRO radius is chosen as $\rho=O(n^{-1/d})$, then the worst-case expected loss over all distributions in the Wasserstein ambiguity set of that radius would be an upper bound for the expected loss with respect to the true data-generating distribution. This provided the first non-asymptotic guarantee for the Wasserstein robust solution, but the bound deteriorates exponentially fast as $d$ grows. 

To address the curse of dimensionality, an empirical likelihood approach was proposed in \citep{blanchet2021sample, blanchet2019confidence,blanchet2019robust} to find the smallest radius $\rho$ such that, with high probability, there exists $Q \in \cP(\R^d)$ with $\Wp(Q, \hat P_n) \leq \rho$ and $\ell^\star \in \cL$ satisfying
\begin{align*}
    \ell^\star \in \argmin_{\ell \in \cL} \E_{Q} [\ell] \cap \argmin_{\ell \in \cL} \E_{P} [\ell]. 
\end{align*}
This choice leads to a confidence region around the optimal solution, which enables working with a radius $\rho=O(n^{-1/2})$. However, this result is only asymptotic in nature, yet finite-sample bounds are crucial for applications. For certain WDROs with linear structure, such as linear regression/classification and kernelized versions thereof, non-asymptotic bounds of $O(n^{-1/2})$, have been established in \citep{shafieezadeh2019regularization,chen2018robust,olea2022generalization,wu2022generalization}. To relax these structural assumptions, \citet{gao2022finite} demonstrated that if the data-generating distribution satisfies a transport-entropy inequality, a radius of $O(n^{-1/2})$  is again sufficient. However, the transport-entropy inequality assumption on the unknown data distribution is restrictive and may be hard to verify in practice. Furthermore, the loss function is required to be $\alpha$-smooth over the family $\cL$ and to admit a sub-root function in order to establish local Rademacher complexity bounds. 

Corollary~\ref{cor:OR-WDRO} present a clean route to overcome the curse of dimensionality in the classical Wasserstein DRO setting, when $\varepsilon = 0$, and obtain finite-sample results without relying on transport inequalities. Moreover, it provides a simple procedure for achieving the excess risk bounds for outlier-robust WDRO presented in \cite{nietert2023outlier} when $p=1$. In fact, the algorithm therein matches our $\sqrt{k\eps} + \rho$ risk bound only when $k = \Theta(1)$ or $k = \Theta(d)$, but not in between. Their approach further requires solving new optimization problems that are more complicated than standard WDRO. Finally, the analysis in that work led to finite-sample excess risk bounds including a term scaling like $n^{-1/d}$ even when $k = O(1)$. The result of \cref{cor:OR-WDRO} accounts for all these limitations, yielding optimal rates uniformly in $k$ via simple and computationally efficient procedures.

\section{Concluding Remarks and Future Work}\label{sec:summary}

In this work, we have provided the first polynomial time algorithm for robust distribution estimation under combined Wasserstein and TV corruptions. In order to apply its guarantees to Wasserstein DRO, we uncovered a practical and conceptually simple technique for alleviating the curse of dimensionality that often manifests itself in this setting. There are numerous directions for future work; some of particular interest include:
\begin{itemize}
    \item For distributions that are $(2\eps,\delta)$-stable and isotropic, we have efficient algorithms for robust mean estimation up to error $\delta$ under TV $\eps$-corruptions. Can we extend these results to obtain $\Wpk$ estimation error $\delta \sqrt{k} + \rho$ under our combined model? For an even simpler challenge, as posed in \cref{ssec:istribution-learning-extensions} --- can one estimate the mean of a spherical Gaussian up to $\ell_2$ error $\tilde{O}(\eps) + \rho$ with both $\Wone$ and TV corruption? We suspect that there may be similar obstacles as those known for robust mean estimation with stable but non-isotropic distributions \citep{hopkins19hard}.
    \item Relatedly, algorithms for robust mean estimation have been refined and optimized in many ways, improving their breakdown points \citep{hopkins2020robust,zhu2022robust,dalalyan2022all}, running time \citep{cheng2019high}, and memory usage \citep{diakonikolas2017being,diakonikolas2022streaming}. We expect many of these improvements to translate to our model.
    \item Finally, for WDRO, can we tractably achieve dependence on the dimensionality $k_\star$ of the optimal loss function if $k_\star \ll k$ (recalling that $k$ is a uniform bound over the loss function family)? We suspect this can be achieved by integrating the objective of \FilterSimple{} into the Wasserstein DRO ambiguity set.
\end{itemize}

\acks{We thank March Boedihardjo for helpful discussion on statistical aspects of sliced Wasserstein distances, in particular for explaining high-quality empirical convergence guarantees under second moment bounds. Z. Goldfeld is partially supported by NSF grants CAREER CCF-2046018, DMS-2210368, and CCF-2308446, and the IBM Academic Award.}

\bibliography{references}

\begin{thebibliography}{67}
\providecommand{\natexlab}[1]{#1}
\providecommand{\url}[1]{\texttt{#1}}
\expandafter\ifx\csname urlstyle\endcsname\relax
  \providecommand{\doi}[1]{doi: #1}\else
  \providecommand{\doi}{doi: \begingroup \urlstyle{rm}\Url}\fi

\bibitem[Angluin and Laird(1988)]{angluin1988learning}
Dana Angluin and Philip Laird.
\newblock Learning from noisy examples.
\newblock \emph{Machine Learning}, 2:\penalty0 343--370, 1988.

\bibitem[Balaji et~al.(2020)Balaji, Chellappa, and Feizi]{balaji2020}
Yogesh Balaji, Rama Chellappa, and Soheil Feizi.
\newblock Robust optimal transport with applications in generative modeling and domain adaptation.
\newblock In \emph{Advances in Neural Information Processing Systems}, 2020.

\bibitem[Bartl and Mendelson(2022)]{bartl2022structure}
Daniel Bartl and Shahar Mendelson.
\newblock Structure preservation via the wasserstein distance.
\newblock \emph{arXiv preprint arXiv:2209.07058}, 2022.

\bibitem[Bennouna and Van~Parys(2022)]{bennouna2022holistic}
Amine Bennouna and Bart Van~Parys.
\newblock Holistic robust data-driven decisions.
\newblock \emph{arXiv preprint arXiv:2207.09560}, 2022.

\bibitem[Bennouna et~al.(2023)Bennouna, Lucas, and Van~Parys]{bennouna2023certified}
Amine Bennouna, Ryan Lucas, and Bart Van~Parys.
\newblock Certified robust neural networks: {G}eneralization and corruption resistance.
\newblock In \emph{International Conference on Machine Learning}, 2023.

\bibitem[Blanchet and Kang(2021)]{blanchet2021sample}
Jose Blanchet and Yang Kang.
\newblock Sample out-of-sample inference based on {W}asserstein distance.
\newblock \emph{Operations Research}, 69\penalty0 (3):\penalty0 985--1013, 2021.

\bibitem[Blanchet and Murthy(2019)]{blanchet2019quantifying}
Jose Blanchet and Karthyek Murthy.
\newblock Quantifying distributional model risk via optimal transport.
\newblock \emph{Mathematics of Operations Research}, 44\penalty0 (2):\penalty0 565--600, 2019.

\bibitem[Blanchet et~al.(2019)Blanchet, Kang, and Murthy]{blanchet2019robust}
Jose Blanchet, Yang Kang, and Karthyek Murthy.
\newblock Robust {W}asserstein profile inference and applications to machine learning.
\newblock \emph{Journal of Applied Probability}, 56\penalty0 (3):\penalty0 830--857, 2019.

\bibitem[Blanchet et~al.(2022)Blanchet, Murthy, and Si]{blanchet2019confidence}
Jose Blanchet, Karthyek Murthy, and Nian Si.
\newblock Confidence regions in {W}asserstein distributionally robust estimation.
\newblock \emph{Biometrika}, 109\penalty0 (2):\penalty0 295--315, 2022.

\bibitem[Boedihardjo(2024)]{boedihardjo2024sharp}
March~T Boedihardjo.
\newblock Sharp bounds for the max-sliced wasserstein distance.
\newblock \emph{arXiv preprint arXiv:2403.00666}, 2024.

\bibitem[Bshouty et~al.(2002)Bshouty, Eiron, and Kushilevitz]{bshouty2002pac}
Nader~H Bshouty, Nadav Eiron, and Eyal Kushilevitz.
\newblock {PAC} learning with nasty noise.
\newblock \emph{Theoretical Computer Science}, 288\penalty0 (2):\penalty0 255--275, 2002.

\bibitem[Caffarelli and McCann(2010)]{caffarelli2010}
Luis~A. Caffarelli and Robert~J. McCann.
\newblock Free boundaries in optimal transport and {M}onge-{A}mp\`ere obstacle problems.
\newblock \emph{Annals of Mathematics. Second Series}, 171\penalty0 (2):\penalty0 673--730, 2010.

\bibitem[Chao and Dobriban(2023)]{chao2023statistical}
Patrick Chao and Edgar Dobriban.
\newblock Statistical estimation under distribution shift: Wasserstein perturbations and minimax theory.
\newblock \emph{arXiv preprint arXiv:2308.01853}, 2023.

\bibitem[Chapel et~al.(2020)Chapel, Alaya, and Gasso]{chapel2020}
Laetitia Chapel, Mokhtar~Z. Alaya, and Gilles Gasso.
\newblock Partial optimal tranport with applications on positive-unlabeled learning.
\newblock In \emph{Advances in Neural Information Processing Systems}, 2020.

\bibitem[Chen and Paschalidis(2018)]{chen2018robust}
Ruidi Chen and Ioannis~Ch Paschalidis.
\newblock A robust learning approach for regression models based on distributionally robust optimization.
\newblock \emph{Journal of Machine Learning Research}, 19\penalty0 (1):\penalty0 517--564, 2018.

\bibitem[Cheng et~al.(2019)Cheng, Diakonikolas, and Ge]{cheng2019high}
Yu~Cheng, Ilias Diakonikolas, and Rong Ge.
\newblock High-dimensional robust mean estimation in nearly-linear time.
\newblock In \emph{SIAM Symposium on Discrete Algorithms}, 2019.

\bibitem[Chizat et~al.(2018{\natexlab{a}})Chizat, Peyr\'{e}, Schmitzer, and Vialard]{chizat2018}
L\'{e}na\"{\i}c Chizat, Gabriel Peyr\'{e}, Bernhard Schmitzer, and Fran\c{c}ois-Xavier Vialard.
\newblock Unbalanced optimal transport: dynamic and {K}antorovich formulations.
\newblock \emph{Journal of Functional Analysis}, 274\penalty0 (11):\penalty0 3090--3123, 2018{\natexlab{a}}.

\bibitem[Chizat et~al.(2018{\natexlab{b}})Chizat, Peyr\'{e}, Schmitzer, and Vialard]{chizat2018scaling}
L\'{e}na\"{\i}c Chizat, Gabriel Peyr\'{e}, Bernhard Schmitzer, and Fran\c{c}ois-Xavier Vialard.
\newblock Scaling algorithms for unbalanced optimal transport problems.
\newblock \emph{Mathematics of Computation}, 87\penalty0 (314):\penalty0 2563--2609, 2018{\natexlab{b}}.

\bibitem[Dalalyan and Minasyan(2022)]{dalalyan2022all}
Arnak~S Dalalyan and Arshak Minasyan.
\newblock All-in-one robust estimator of the gaussian mean.
\newblock \emph{The Annals of Statistics}, 50\penalty0 (2):\penalty0 1193--1219, 2022.

\bibitem[Diakonikolas and Kane(2023)]{diakonikolas2023algorithmic}
Ilias Diakonikolas and Daniel~M Kane.
\newblock \emph{Algorithmic High-Dimensional Robust Statistics}.
\newblock Cambridge University Press, 2023.

\bibitem[Diakonikolas et~al.(2016)Diakonikolas, Kamath, Kane, Li, Moitra, and Stewart]{diakonikolas2016robust}
Ilias Diakonikolas, Gautam Kamath, Daniel~M Kane, Jerry Li, Ankur Moitra, and Alistair Stewart.
\newblock Robust estimators in high dimensions without the computational intractability.
\newblock In \emph{IEEE Symposium on Foundations of Computer Science}, 2016.

\bibitem[Diakonikolas et~al.(2017)Diakonikolas, Kamath, Kane, Li, Moitra, and Stewart]{diakonikolas2017being}
Ilias Diakonikolas, Gautam Kamath, Daniel~M Kane, Jerry Li, Ankur Moitra, and Alistair Stewart.
\newblock Being robust (in high dimensions) can be practical.
\newblock In \emph{International Conference on Machine Learning}, 2017.

\bibitem[Diakonikolas et~al.(2018)Diakonikolas, Kamath, Kane, Li, Moitra, and Stewart]{diakonikolas2018robustly}
Ilias Diakonikolas, Gautam Kamath, Daniel~M Kane, Jerry Li, Ankur Moitra, and Alistair Stewart.
\newblock Robustly learning a gaussian: Getting optimal error, efficiently.
\newblock In \emph{Proceedings of the Twenty-Ninth Annual ACM-SIAM Symposium on Discrete Algorithms}, 2018.

\bibitem[Diakonikolas et~al.(2022)Diakonikolas, Kane, Pensia, and Pittas]{diakonikolas2022streaming}
Ilias Diakonikolas, Daniel~M Kane, Ankit Pensia, and Thanasis Pittas.
\newblock Streaming algorithms for high-dimensional robust statistics.
\newblock In \emph{International Conference on Machine Learning}, 2022.

\bibitem[Donoho and Liu(1988)]{donoho88}
David~L. Donoho and Richard~C. Liu.
\newblock {The "Automatic" Robustness of Minimum Distance Functionals}.
\newblock \emph{The Annals of Statistics}, 16\penalty0 (2):\penalty0 552 -- 586, 1988.

\bibitem[Esteban-P{\'e}rez and Morales(2022)]{esteban2022distributionally}
Adri{\'a}n Esteban-P{\'e}rez and Juan~M Morales.
\newblock Distributionally robust stochastic programs with side information based on trimmings.
\newblock \emph{Mathematical Programming}, 195\penalty0 (1-2):\penalty0 1069--1105, 2022.

\bibitem[Fatras et~al.(2021)Fatras, Sejourne, Flamary, and Courty]{fatras21a}
Kilian Fatras, Thibault Sejourne, R{\'e}mi Flamary, and Nicolas Courty.
\newblock Unbalanced minibatch optimal transport; applications to domain adaptation.
\newblock In \emph{International Conference on Machine Learning}, 2021.

\bibitem[Figalli(2010)]{figalli2010}
Alessio Figalli.
\newblock The optimal partial transport problem.
\newblock \emph{Archive for Rational Mechanics and Analysis}, 195:\penalty0 533--560, 2010.

\bibitem[Fournier and Guillin(2015)]{fournier2015rate}
Nicolas Fournier and Arnaud Guillin.
\newblock On the rate of convergence in {W}asserstein distance of the empirical measure.
\newblock \emph{Probability Theory and Related Fields}, 162\penalty0 (3-4):\penalty0 707--738, 2015.

\bibitem[Fukunaga and Kasai(2022)]{fukunaga2021}
Takumi Fukunaga and Hiroyuki Kasai.
\newblock Block-coordinate {F}rank-{W}olfe algorithm and convergence analysis for semi-relaxed optimal transport problem.
\newblock In \emph{IEEE International Conference on Acoustics, Speech and Signal Processing}, 2022.

\bibitem[Gao(2022)]{gao2022finite}
Rui Gao.
\newblock Finite-sample guarantees for {W}asserstein distributionally robust optimization: {B}reaking the curse of dimensionality.
\newblock \emph{Operations Research}, 2022.

\bibitem[Gao and Kleywegt(2023)]{gao2023distributionally}
Rui Gao and Anton Kleywegt.
\newblock Distributionally robust stochastic optimization with {W}asserstein distance.
\newblock \emph{Mathematics of Operations Research}, 48\penalty0 (2):\penalty0 603--655, 2023.

\bibitem[Hanin(1992)]{hanin1992}
Leonid~G. Hanin.
\newblock Kantorovich-{R}ubinstein norm and its application in the theory of {L}ipschitz spaces.
\newblock \emph{Proceedings of the American Mathematical Society}, 115\penalty0 (2):\penalty0 345--352, 1992.

\bibitem[Hashimoto et~al.(2018)Hashimoto, Srivastava, Namkoong, and Liang]{hashimoto2018fairness}
Tatsunori Hashimoto, Megha Srivastava, Hongseok Namkoong, and Percy Liang.
\newblock Fairness without demographics in repeated loss minimization.
\newblock In \emph{International Conference on Machine Learning}, 2018.

\bibitem[Hopkins et~al.(2020)Hopkins, Li, and Zhang]{hopkins2020robust}
Sam Hopkins, Jerry Li, and Fred Zhang.
\newblock Robust and heavy-tailed mean estimation made simple, via regret minimization.
\newblock \emph{Advances in Neural Information Processing Systems}, 33, 2020.

\bibitem[Hopkins and Li(2019)]{hopkins19hard}
Samuel~B. Hopkins and Jerry Li.
\newblock How hard is robust mean estimation?
\newblock In \emph{Conference on Learning Theory}, 2019.

\bibitem[Hu et~al.(2018)Hu, Niu, Sato, and Sugiyama]{hu2018does}
Weihua Hu, Gang Niu, Issei Sato, and Masashi Sugiyama.
\newblock Does distributionally robust supervised learning give robust classifiers?
\newblock In \emph{International Conference on Machine Learning}, 2018.

\bibitem[Huber(1964)]{huber64}
Peter~J. Huber.
\newblock {Robust Estimation of a Location Parameter}.
\newblock \emph{The Annals of Mathematical Statistics}, 35\penalty0 (1):\penalty0 73--101, 1964.

\bibitem[Klivans et~al.(2009)Klivans, Long, and Servedio]{klivans2009learning}
Adam~R Klivans, Philip~M Long, and Rocco~A Servedio.
\newblock Learning halfspaces with malicious noise.
\newblock \emph{Journal of Machine Learning Research}, 10\penalty0 (12), 2009.

\bibitem[Le et~al.(2021)Le, Nguyen, Nguyen, Pham, Bui, and Ho]{le2021}
Khang Le, Huy Nguyen, Quang~M Nguyen, Tung Pham, Hung Bui, and Nhat Ho.
\newblock On robust optimal transport: {C}omputational complexity and barycenter computation.
\newblock In \emph{Advances in Neural Information Processing Systems}, 2021.

\bibitem[LEI(2020)]{lei2020convergence}
JING LEI.
\newblock Convergence and concentration of empirical measures under {W}asserstein distance in unbounded functional spaces.
\newblock \emph{Bernoulli}, 26\penalty0 (1):\penalty0 767--798, 2020.

\bibitem[Liero et~al.(2018)Liero, Mielke, and Savar\'{e}]{liero2018}
Matthias Liero, Alexander Mielke, and Giuseppe Savar\'{e}.
\newblock Optimal entropy-transport problems and a new {H}ellinger-{K}antorovich distance between positive measures.
\newblock \emph{Inventiones Mathematicae}, 211\penalty0 (3):\penalty0 969--1117, 2018.

\bibitem[Lin et~al.(2020)Lin, Fan, Ho, Cuturi, and Jordan]{lin2020projection}
Tianyi Lin, Chenyou Fan, Nhat Ho, Marco Cuturi, and Michael Jordan.
\newblock Projection robust {W}asserstein distance and {R}iemannian optimization.
\newblock In \emph{Advances in Neural Information Processing Systems}, 2020.

\bibitem[Lin et~al.(2021)Lin, Zheng, Chen, Cuturi, and Jordan]{lin2021projection}
Tianyi Lin, Zeyu Zheng, Elynn Chen, Marco Cuturi, and Michael~I Jordan.
\newblock On projection robust optimal transport: {S}ample complexity and model misspecification.
\newblock In \emph{International Conference on Artificial Intelligence and Statistics}, 2021.

\bibitem[Mohajerin~Esfahani and Kuhn(2018)]{mohajerin2018data}
Peyman Mohajerin~Esfahani and Daniel Kuhn.
\newblock Data-driven distributionally robust optimization using the {W}asserstein metric: {P}erformance guarantees and tractable reformulations.
\newblock \emph{Mathematical Programming}, 171\penalty0 (1-2):\penalty0 115--166, 2018.

\bibitem[Mukherjee et~al.(2021)Mukherjee, Guha, Solomon, Sun, and Yurochkin]{mukherjee2021}
Debarghya Mukherjee, Aritra Guha, Justin Solomon, Yuekai Sun, and Mikhail Yurochkin.
\newblock Outlier-robust optimal transport.
\newblock In \emph{International Conference on Machine Learning}, 2021.

\bibitem[Nadjahi et~al.(2020)Nadjahi, Durmus, Chizat, Kolouri, Shahrampour, and Simsekli]{nadjahi2020statistical}
Kimia Nadjahi, Alain Durmus, L{\'e}na{\"\i}c Chizat, Soheil Kolouri, Shahin Shahrampour, and Umut Simsekli.
\newblock Statistical and topological properties of sliced probability divergences.
\newblock In \emph{Advances in Neural Information Processing Systems}, 2020.

\bibitem[Nath(2020)]{nath2020}
J.~Saketha Nath.
\newblock Unbalanced optimal transport using integral probability metric regularization.
\newblock \emph{arXiv preprint arXiv:2011.05001}, 2020.

\bibitem[Nietert et~al.(2021)Nietert, Cummings, and Goldfeld]{nietert2021outlier}
Sloan Nietert, Rachel Cummings, and Ziv Goldfeld.
\newblock Outlier-robust optimal transport with applications to generative modeling and data privacy.
\newblock In \emph{Theory and Practice of Differential Privacy Workshop at ICML}, 2021.

\bibitem[Nietert et~al.(2022{\natexlab{a}})Nietert, Goldfeld, Sadhu, and Kato]{nietert2022statistical}
Sloan Nietert, Ziv Goldfeld, Ritwik Sadhu, and Kengo Kato.
\newblock Statistical, robustness, and computational guarantees for sliced wasserstein distances.
\newblock In \emph{Advances in Neural Information Processing Systems}, 2022{\natexlab{a}}.

\bibitem[Nietert et~al.(2022{\natexlab{b}})Nietert, Sadhu, Goldfeld, and Kato]{nietert2022sliced}
Sloan Nietert, Ritwik Sadhu, Ziv Goldfeld, and Kengo Kato.
\newblock Statistical, robustness, and computational guarantees for sliced {W}asserstein distances.
\newblock In \emph{Advances in Neural Information Processing Systems}, 2022{\natexlab{b}}.

\bibitem[Nietert et~al.(2023{\natexlab{a}})Nietert, Cummings, and Goldfeld]{nietert2023robust}
Sloan Nietert, Rachel Cummings, and Ziv Goldfeld.
\newblock Robust estimation under the {W}asserstein distance.
\newblock \emph{arXiv preprint arXiv:2302.01237}, 2023{\natexlab{a}}.

\bibitem[Nietert et~al.(2023{\natexlab{b}})Nietert, Goldfeld, and Shafiee]{nietert2023outlier}
Sloan Nietert, Ziv Goldfeld, and Soroosh Shafiee.
\newblock Outlier-robust {W}asserstein {DRO}.
\newblock In \emph{Advances in Neural Information Processing Systems}, 2023{\natexlab{b}}.

\bibitem[Niles-Weed and Rigollet(2022)]{niles2022estimation}
Jonathan Niles-Weed and Philippe Rigollet.
\newblock Estimation of {W}asserstein distances in the spiked transport model.
\newblock \emph{Bernoulli}, 28\penalty0 (4):\penalty0 2663--2688, 2022.

\bibitem[Olea et~al.(2022)Olea, Rush, Velez, and Wiesel]{olea2022generalization}
Jos{\'e} Luis~Montiel Olea, Cynthia Rush, Amilcar Velez, and Johannes Wiesel.
\newblock On the generalization error of norm penalty linear regression models.
\newblock \emph{arXiv preprint arXiv:2211.07608}, 2022.

\bibitem[Piccoli and Rossi(2014)]{piccoli2014}
Benedetto Piccoli and Francesco Rossi.
\newblock Generalized {W}asserstein distance and its application to transport equations with source.
\newblock \emph{Archive for Rational Mechanics and Analysis}, 211\penalty0 (1):\penalty0 335--358, 2014.

\bibitem[Ronchetti and Huber(2009)]{ronchetti2009robust}
Elvezio~M Ronchetti and Peter~J Huber.
\newblock \emph{Robust Statistics}.
\newblock John Wiley \& Sons Hoboken, 2009.

\bibitem[Santambrogio(2015)]{santambrogio2015}
Filippo Santambrogio.
\newblock \emph{Optimal Transport for Applied Mathematicians}.
\newblock Springer, 2015.

\bibitem[Schmitzer and Wirth(2019)]{schmitzer2019}
Bernhard Schmitzer and Benedikt Wirth.
\newblock A framework for {W}asserstein-1-type metrics.
\newblock \emph{Journal of Convex Analysis}, 26\penalty0 (2):\penalty0 353--396, 2019.

\bibitem[Shafieezadeh-Abadeh et~al.(2019)Shafieezadeh-Abadeh, Kuhn, and Mohajerin~Esfahani]{shafieezadeh2019regularization}
Soroosh Shafieezadeh-Abadeh, Daniel Kuhn, and Peyman Mohajerin~Esfahani.
\newblock Regularization via mass transportation.
\newblock \emph{Journal of Machine Learning Research}, 20\penalty0 (103):\penalty0 1--68, 2019.

\bibitem[Sinha et~al.(2018)Sinha, Namkoong, and Duchi]{sinha2018certifying}
Aman Sinha, Hongseok Namkoong, and John Duchi.
\newblock Certifying some distributional robustness with principled adversarial training.
\newblock In \emph{International Conference on Learning Representations}, 2018.

\bibitem[Steinhardt et~al.(2018)Steinhardt, Charikar, and Valiant]{steinhardt2018resilience}
Jacob Steinhardt, Moses Charikar, and Gregory Valiant.
\newblock Resilience: {A} criterion for learning in the presence of arbitrary outliers.
\newblock In \emph{Innovations in Theoretical Computer Science Conference}, volume~94, 2018.

\bibitem[Villani(2003)]{villani2003}
C\'{e}dric Villani.
\newblock \emph{Topics in Optimal Transportation}.
\newblock Graduate Studies in Mathematics. American Mathematical Society, 2003.

\bibitem[Wu et~al.(2022)Wu, Li, and Mao]{wu2022generalization}
Qinyu Wu, Jonathan Yu-Meng Li, and Tiantian Mao.
\newblock On generalization and regularization via {W}asserstein distributionally robust optimization.
\newblock \emph{arXiv preprint arXiv:2212.05716}, 2022.

\bibitem[Zhai et~al.(2021)Zhai, Dan, Kolter, and Ravikumar]{zhai2021doro}
Runtian Zhai, Chen Dan, Zico Kolter, and Pradeep Ravikumar.
\newblock {DORO}: {D}istributional and outlier robust optimization.
\newblock In \emph{International Conference on Machine Learning}, 2021.

\bibitem[Zhu et~al.(2022{\natexlab{a}})Zhu, Jiao, and Steinhardt]{zhu2019resilience}
Banghua Zhu, Jiantao Jiao, and Jacob Steinhardt.
\newblock {Generalized resilience and robust statistics}.
\newblock \emph{The Annals of Statistics}, 50\penalty0 (4):\penalty0 2256 -- 2283, 2022{\natexlab{a}}.

\bibitem[Zhu et~al.(2022{\natexlab{b}})Zhu, Jiao, and Steinhardt]{zhu2022robust}
Banghua Zhu, Jiantao Jiao, and Jacob Steinhardt.
\newblock Robust estimation via generalized quasi-gradients.
\newblock \emph{Information and Inference: A Journal of the IMA}, 11\penalty0 (2):\penalty0 581--636, 2022{\natexlab{b}}.

\end{thebibliography}

\appendix
\crefalias{section}{appendix}

\section{Proofs for \cref{ssec:distribution-learning-population-limit}}
\label{app:distribution-learning-population-limit}

Throughout this section, we prove results under a more general learning environment.

\begin{tcolorbox}[colback=white]
\label{setting:A2}
\textbf{Setting A2:} Fix TV corruption level $0 \leq \eps \leq 0.49$ and $\Wp$ corruption level $\rho \geq 0$, where $p \in \{1,2\}$. Let $\cG \subseteq \cP(\R^d)$. Nature selects $P \in \cG$, and the learner observes $\tilde{P}$ such that $\RWp(\tilde{P},P) \leq \rho$.
\end{tcolorbox}

We begin with some auxiliary definitions and lemmas.

\begin{definition}[Resilience]
For $P \in \cP(\R^d)$ and $0 \leq \eps < 1$, the mean $\eps$-resilience of $P$ is given by 
\begin{equation*}
    \tau(P,\eps) \defeq \sup_{Q \in \cP(\R^d) : Q \leq \frac{1}{1-\eps}P} \|\mu_Q - \mu_P\|.
\end{equation*}
For $p \geq 1$, the $p$th-order $\eps$-resilience of $P \in \cP_p(\R^d)$ is defined by $\tau_p(P,\eps) \defeq \tau(f_\sharp  P,\eps)$, where $f(z) = \|z - \mu_P\|^p$. For a family $\cG \subseteq \cP(\R^d)$, we define $\tau(\cG,\eps) \defeq \sup_{P \in \cG}\tau(P,\eps)$ and $\tau_p(\cG,\eps) \defeq \sup_{P \in \cG}\tau_p(P,\eps)$.
\end{definition}

It generally suffices to analyze resilience for $\eps$ bounded away from 1, due to the following result.

\begin{lemma}
\label[lemma]{lem:large-eps-resilience}
For each $P \in \cP(\R^d)$ and $0 < \eps < 1$, we have $\tau(P,1-\eps) = \frac{1-\eps}{\eps} \tau(P,\eps)$.
\end{lemma}
\begin{proof}\ \ \ \  
If $P = (1-\eps)Q + \eps R$ for $Q,R \in \cP(\R^d)$, we have
\begin{align*}
    \|\mu_P - \mu_R\| = \frac{1-\eps}{\eps}\|\mu_P - \mu_Q\| \leq \frac{1-\eps}{\eps} \tau(P,\eps).
\end{align*}
Supremizing over $R$ gives one direction, and substituting $\eps \gets 1-\eps$ gives the other.
\end{proof}

We also observe a certain monotonicity of $p$th moment resilience terms in $p$.

\begin{lemma}
\label[lemma]{lem:resilience-monotonicity}
Fix $1 \leq p < q$, $0 \leq \eps < 1$, and $P \in \cP(\R^d)$. Then $\tau(P,\eps) \leq \tau_p(P,\eps)^{1/p} \leq \tau_q(P,\eps)^{1/q}$.
\end{lemma}
\begin{proof}\ \ \ \  
Take $X \sim P$ and $Y \sim Q$, for any $Q \in \cP(\R^d)$ such that $Q \leq \frac{1}{1-\eps}P$. We then bound
\begin{align*}
    \|\mu_Q - \mu_P\| \leq \E[\|Y - X\|] \leq \bigl|\E[\|Y - \mu_P\| - \|X - \mu_P\|]\bigr| \leq \tau_1(P,\eps).
\end{align*}
Moreover, writing $a = \E[\|X - \mu_P\|^p]$, $b = \E[\|Y - \mu_P\|^p]$, and $r = q/p \geq 1$, we have
\begin{align*}
   \bigl|\E[\|Y - \mu_P\|^p - \|X - \mu_P\|^p]\bigr| \leq |a - b| \leq |a^r - b^r|^{1/r} = \bigl|\E[\|Y - \mu_P\|^q - \|X - \mu_P\|^q]\bigr|^{p/q}.
\end{align*}
Raising both sides to the $(1/p)$th power and supremizing over $Q$ completes the proof.
\end{proof}

Stability essentially captures resilience in first and second moments.

\begin{lemma}
\label{lem:resilience-from-stability}
Let $0 < \eps < 1$ and $\delta \geq \eps$. For $P \in \cS(\eps,\delta)$, we have $\tau(P,\eps) \leq \delta$ and $\tau_2(P,\eps) \leq 2d \delta^2/\eps$.
\end{lemma}
\begin{proof}\ \ \ \  
Mean resilience follows directly from the definition of $(\eps,\delta)$-stability. For second moment resilience, we fix any $Q \leq \frac{1}{1-\eps}P$ and bound
\begin{align*}
    \bigl|\tr(\Sigma_Q(\mu_P)) - \tr(\Sigma_P)\bigr| = |\tr(\Sigma_Q - \Sigma_P)| + \|\mu_Q - \mu_P\|^2 \leq \frac{\delta^2}{\eps} + \delta^2 \leq \frac{2\delta^2}{\eps}.
\end{align*}
Supremizing over $Q$ gives the lemma.
\end{proof}

Next, we compute useful resilience bounds for distributions that lie near $\cS(\eps,\delta)$ under $\Wtwo$. For brevity, we define $\cS(\eps,\delta,\lambda) \defeq \{ P \in \cP(\R^d) : \Wtwo(p,\cS(\eps,\delta)) \leq \lambda \}$.

\begin{lemma}
\label[lemma]{lem:resilience-from-near-stability}
Let $0 < \eps < 1$, $\delta \geq \eps$, and $\lambda \geq 0$. For $P \in \cS(\eps,\delta,\lambda)$, we have
\begin{align*}
    \tau(P,\eps) &\leq \delta + \frac{2\lambda\sqrt{\eps}}{1-\eps}\\
    \tau_2(P,\eps) &\leq \frac{4d\delta^2}{(1-\eps)\eps} + \frac{16 \lambda^2}{1-\eps}
\end{align*}
Finally, we have $\tr(\Sigma_P) \leq 2d + 4\lambda^2$ and, for any $Q \leq \frac{1}{1-\eps}P$,
\begin{equation*}
    \tr(\Sigma_Q) \leq \frac{6d\delta^2}{(1-\eps)\eps^2} + \frac{20\lambda^2}{1-\eps}.
\end{equation*}
\end{lemma}
\begin{proof}\ \ \ \  
Fix $P_0 \in \cS(\eps,\delta)$ such that $\Wtwo(P_0,P) \leq \lambda$, and let $X,Y$ be optimal coupling for $\Wtwo(P_0,P)$. Fix any $Q \leq \frac{1}{1-\eps}P$. Augmenting the probability space if necessary, we can realize $Q$ as the law of $Y$ conditioned on an event $E$ with probability $1-\eps$. Write $P'_0$ for the law of $X$ conditioned on $E$. We then bound
\begin{align*}
    \|\mu_Q - \mu_P\| &= \frac{\eps}{1-\eps} \|\E[Y] - \E[Y|E^c]\|\\ &\leq \frac{\eps}{1-\eps} \left(\|\E[X] - \E[X|E^c]\| + \|\E[Y] - \E[X]\|  + \|\E[Y|E^c] - \E[X|E^c]\| \right)\\
    &\leq \frac{\eps}{1-\eps} \left(\|\E[X] - \E[X|E^c]\| + \Wtwo(P_0,P) + \frac{1}{\sqrt{\eps}}\Wtwo(P_0,P) \right)\\
    &\leq \|\mu_{P_0} - \mu_{P_0'}\| + \frac{2\sqrt{\eps}}{1-\eps}\lambda\\
    &\leq \delta + \frac{2\sqrt{\eps}}{1-\eps}\lambda.
\end{align*}
Supremizing over $Q$ gives the mean resilience bound.
Next, we use Minkowski's inequality to bound
\begin{align*}
    \bigl|\tr\bigl(\Sigma_P - \Sigma_{P_0}\bigr)\bigr| &= \bigl|\E[\|Y - \E[Y]\|^2] - \tr(\Sigma_{P_0})\bigl|\\
    &\leq \left|\left(\E[\|X - \E[Y]\|^2]^\frac{1}{2} + \E[\|Y - X\|^2]^\frac{1}{2} + \|\E[X] - \E[Y]\|\right)^2 - \tr(\Sigma_{P_0})\right|\\
    &\leq \left|\left(\sqrt{\tr(\Sigma_{P_0})} + 2\lambda\right)^2 - \tr(\Sigma_{P_0})\right|\\
    &\leq 4\lambda \sqrt{\tr(\Sigma_{P_0})} + 4\lambda^2\\
    &\leq 4\lambda\sqrt{d} + 4\lambda^2, \tag{$\Sigma_{P_0} \preceq I_d$}
\end{align*}
which implies the desired bound on $\tr(\Sigma_P)$. The same argument via Minkowski's inequality gives
\begin{align*}
   \bigl|\E[\|X - \E[X]\|^2|E^c] - \E[\|Y - \E[Y]\|^2 |E^c]\bigl| \leq \frac{4\lambda\sqrt{d}}{\sqrt{\eps}} + \frac{4\lambda^2}{\eps}.
\end{align*}
We then compute
\begin{align*}
    \bigl|\tr(\Sigma_{Q}(\mu_{P}) - \Sigma_{P})\bigr| &= \left|\E\bigl[\|Y - \E[Y]\|^2 \bigm| E\bigr] - \E\bigl[\|Y - \E[Y]\|^2\bigr]\right| \\
    &= \frac{\eps}{1-\eps}  \left|\E\bigl[\|Y - \E[Y]\|^2 \bigr] - \E\bigl[\|Y - \E[Y]\|^2 \bigm| E^c\bigr]\right|\\
    &\leq \frac{\eps}{1-\eps} \left( \left|\E\bigl[\|X - \E[X]\| \bigr] - \E\bigl[\|X - \E[X]\| \bigm| E^c\bigr]\right| + \frac{8\lambda\sqrt{d}}{\sqrt{\eps}} + \frac{8\lambda^2}{\eps}\right)\\
    &= \left|\E\bigl[\|X - \E[X]\| \bigm| E\bigr] - \E\bigl[\|X - \E[X]\|\bigr]\right| + \frac{8\lambda\sqrt{\eps d}}{1-\eps} + \frac{8\lambda^2}{1-\eps}\\
    &\leq \tau_2(P_0,\eps) + \frac{8\lambda\sqrt{\eps d}}{1-\eps} + \frac{8\lambda^2}{1-\eps}\\
    &\leq \frac{2d\delta^2}{\eps} + \frac{8\lambda\sqrt{\eps d}}{1-\eps} + \frac{8\lambda^2}{1-\eps} \tag{\cref{lem:resilience-from-stability}}\\
    &\leq \frac{1}{1-\eps} \left(\frac{\sqrt{2d}\delta}{\sqrt{\eps}} + \sqrt{8}\lambda\right)^2 \tag{$\delta \geq \eps$}\\
    &\leq \frac{4d\delta^2}{(1-\eps)\eps} + \frac{16 \lambda^2}{1-\eps}
\end{align*}
Supremizing over $Q$ gives the second moment resilience bound. Finally, we bound
\begin{align*}
    \tr(\Sigma_Q) &= \tr(\Sigma_P) + \tr(\Sigma_Q(\mu_P) - \Sigma_P) - \|\mu_P - \mu_Q\|^2\\
    &\leq \tr(\Sigma_P) + \tau_2(P,\eps)\\
    &\leq 2d + 4\lambda^2 + \frac{4d\delta^2}{(1-\eps)\eps} + \frac{16 \lambda^2}{1-\eps}\\
    &\leq \frac{6d\delta^2}{(1-\eps)\eps^2} + \frac{20\lambda^2}{1-\eps},
\end{align*}
as desired.

\end{proof}

Finally, we prove a technical lemma used throughout.

\begin{lemma}
\label[lemma]{lem:tv-regularizer-helper}
Fix $0 < \eps < 1$ and $P \in \cS(\eps,\delta,\lambda)$. Suppose that $Q = (1-\eps)P' + \eps R$ for some $P',R \in \cP(\R^d)$ such that $P' \leq \frac{1}{1-\eps}P$. Then, for $1 \leq q \leq 2$, we have
\begin{align*}
    \Wq(P,Q) \leq \frac{7\eps^{\frac{1}{q}-1}\delta \sqrt{d}}{\sqrt{1-\eps}} + \frac{12 \eps^{\frac{1}{q}-\frac{1}{2}}\lambda}{\sqrt{1-\eps}} + 2 \eps^\frac{1}{q} \sqrt{\tr(\Sigma_R(\mu_P))}.
\end{align*}
\end{lemma}
\begin{proof}\ \ \ \  
Write $P = (1-\eps)P' + \eps S$, for some $S \in \cP(\R^d)$. For any $q \in [1,2]$, we have
\begin{align*}
    \Wq(P,Q)^q &\leq \eps \Wq(S,R)^q\\
    &\leq 2^q \eps \bigl(\Wq(S,\delta_{\mu_P})^q + \Wq(\delta_{\mu_P},R)^q\bigr)\\
    &\leq 2^q \eps\bigl(\Wq(P,\delta_{\mu_P})^q + \tau_q(P,1-\eps) + \Wq(\delta_{\mu_P},R)^q \bigr).
\end{align*}
Taking $p$th roots and applying \cref{lem:resilience-from-near-stability}, we obtain
\begin{align*}
    \Wq(P,Q) &\leq 2 \eps^\frac{1}{q} \Wq(P,\delta_{\mu_P}) + 2 \eps^\frac{1}{q} \tau_q(P,1-\eps)^\frac{1}{q} + 2 \eps^\frac{1}{q} \Wq(\delta_{\mu_P},R)\\
    &\leq 2 \eps^\frac{1}{q} \Wtwo(P,\delta_{\mu_P}) + 2 \eps^\frac{1}{q} \sqrt{\tau_2(P,1-\eps)} + 2 \eps^\frac{1}{q}  \Wtwo(\delta_{\mu_P},R)\\
    &\leq 2 \eps^\frac{1}{q} \sqrt{\tr(\Sigma_P)} + 2 \eps^{\frac{1}{q} - \frac{1}{2}} \sqrt{\tau_2(P,\eps)} + 2 \eps^\frac{1}{q} \sqrt{\tr(\Sigma_R(\mu_P))}\\
    &\leq 2 \eps^\frac{1}{q} \sqrt{2d + 4\lambda^2} + 2 \eps^{\frac{1}{q} - \frac{1}{2}} \sqrt{\frac{4d\delta^2}{(1-\eps)\eps} + \frac{16 \lambda^2}{1-\eps}} + 2 \eps^\frac{1}{q} \sqrt{\tr(\Sigma_R(\mu_P))}\\
    &\leq \frac{7\eps^{\frac{1}{q}-1}\delta \sqrt{d}}{\sqrt{1-\eps}} + \frac{12 \eps^{\frac{1}{q}-\frac{1}{2}}\lambda}{\sqrt{1-\eps}} + 2 \eps^\frac{1}{q} \sqrt{\tr(\Sigma_R(\mu_P))},
\end{align*}
as desired.
\end{proof}

As a consequence, we can bound the Wasserstein distance between a stable distribution and any of its $\eps$-deletions. In \cite{nietert2023robust}, this is called a ``Wasserstein resilience'' bound.

\begin{lemma}[Wasserstein resilience from stability]
\label[lemma]{lem:Wq-resilience-under-stability}
Fix $0 < \eps < 1$, $P \in \cS(\eps,\delta,\lambda)$, and $Q \leq \frac{1}{1-\eps}P$. Then, for $1 \leq q \leq 2$, we have
\begin{align*}
    \Wq(P,Q) \leq \frac{12\eps^{\frac{1}{q}-1}\delta \sqrt{d}}{\sqrt{1-\eps}} + \frac{21 \eps^{\frac{1}{q}-\frac{1}{2}}\lambda}{\sqrt{1-\eps}}.
\end{align*}
\end{lemma}
\begin{proof}\ \ \ \  
Writing $P = (1-\eps)Q + \eps R$ for some $R \in \cP(\R^d)$, we use \cref{lem:resilience-from-near-stability} to bound
\begin{align*}
    \tr(\Sigma_Q(\mu_P)) &\leq \tr(\Sigma_P) + \tau_2(P,\eps)\\
    &\leq 2d + 4\lambda^2 + \frac{4d\delta^2}{(1-\eps)\eps} + \frac{16 \lambda^2}{1-\eps}\\
    &\leq \frac{6d\delta^2}{(1-\eps)\eps^2} + \frac{20\lambda^2}{1-\eps}.
\end{align*}
Noting that $Q = (1-\eps)Q + \eps Q$ and $Q \leq \frac{1}{1-\eps}P$, we use \cref{lem:tv-regularizer-helper} to bound
\begin{align*}
     \Wq(P,Q) &\leq \frac{7\eps^{\frac{1}{q}-1}\delta \sqrt{d}}{\sqrt{1-\eps}} + \frac{12 \eps^{\frac{1}{q}-\frac{1}{2}}\lambda}{\sqrt{1-\eps}} + 2 \eps^\frac{1}{q} \sqrt{\tr(\Sigma_Q(\mu_P))}\\
     &\leq \frac{7\eps^{\frac{1}{q}-1}\delta \sqrt{d}}{\sqrt{1-\eps}} + \frac{12 \eps^{\frac{1}{q}-\frac{1}{2}}\lambda}{\sqrt{1-\eps}} + 2\eps^\frac{1}{q}\sqrt{\frac{6d\delta^2}{(1-\eps)\eps^2} + \frac{20\lambda^2}{1-\eps}}\\
     &\leq \frac{12\eps^{\frac{1}{q}-1}\delta \sqrt{d}}{\sqrt{1-\eps}} + \frac{21 \eps^{\frac{1}{q}-\frac{1}{2}}\lambda}{\sqrt{1-\eps}},
\end{align*}
as desired.
\end{proof}

\subsection{Proof of bounds in \cref{ex:stability-bounds}}
\label{prf:stability-bounds}

We first observe that a distribution is $(\eps,\delta)$-stable if and only if all of its 1-dimensional orthogonal projections are $(\eps,\delta)$-stable. We thus assume that $d=1$ without loss of generality.\medskip

Next we prove a useful stability bound under the more general condition of an Orlicz norm bound. Recall that an Orlicz function is any convex, non-decreasing function $\psi:\R_+ \to \R_+$ such that $\psi(0) = 0$ and $\psi(x) \to \infty$ as $x \to \infty$. For a real random variable $X$, we define its Orlicz norm with respect to $\psi$ by $\|X\|_{\psi} \defeq \sup \{ \sigma \geq 0 : \E[\psi(|X|/\sigma)] \leq 1\}$.

\begin{lemma}
\label[lemma]{lem:stability-under-orlicz-norm}
Suppose that $\|X - \E[X]\|_\psi \leq \sigma$, where $\psi$ is an Orlicz function satisfying $\psi(x) = \phi(x^2)$ for another Orlicz function $\phi$. Then $X$ is $\bigl(\eps,O(\sigma\eps\psi^{-1}(1/\eps))\bigr)$-stable for $0 \leq \eps \leq 0.99$.
\end{lemma}
\begin{proof}\ \ \ \  
We assume without loss of generality that $\E[X] = 0$. For mean resilience, take $E$ to be any event with probability $1-\eps' \geq 1-\eps$, and bound
\begin{align*}
    |\E[X|E]| &= \frac{\eps'}{1-\eps'}|\E[X|E^c]|\\
    &\leq \frac{\eps'}{1-\eps'}\E\bigl[|X|\bigm|E^c\bigr]\\
    &\leq \frac{\sigma\eps'}{1-\eps'}\psi^{-1}\left(\E\left[\psi\left(\frac{|X|}{\sigma}\right)\biggm|E^c\right]\right)\\
    &\leq \frac{\sigma\eps'}{1-\eps'}\psi^{-1}\left(\frac{1}{\eps'}\right)\\
    &\leq \frac{\sigma\eps}{1-\eps}\psi^{-1}\left(\frac{1}{\eps}\right) \eqqcolon \delta.
\end{align*}
We note that this approach is well-known; see, e.g., Lemma E.2 of \cite{zhu2019resilience}.
For the second moment condition, we apply the bound above to $X^2$ (with $\|X^2\|_\phi \leq \sigma^2$), deducing
\begin{align*}
    |\E[X^2|E] - \E[X^2]| &\leq \frac{\sigma^2\eps}{1-\eps}\phi^{-1}\left(\frac{1}{\eps}\right)\\
    &\leq \frac{\sigma^2\eps}{1-\eps}\psi^{-1}\left(\frac{1}{\eps}\right)^2\\
    &\leq \frac{\delta^2}{\eps}.
\end{align*}
Combining, we bound $|\!\Var[X|E] - \Var[X]| \leq \delta^2/\eps + \E[X|E]^2 \leq 2\delta^2/\eps$. Thus, $X$ is $(\eps,\sqrt{2}\delta)$-stable.
\end{proof}

We now turn to the specific examples.

\paragraph{Bounded covariance:} We apply the lemma with $\psi(x) = x^2$ and $\phi(x) = x$, giving $(\eps,O(\sqrt{\eps}))$-stability. The near-isotropic restriction is vacuous since $\delta \gtrsim \sqrt{\eps}$.

\paragraph{Sub-Gaussian:} We apply the lemma with $\psi(x) = \exp(x^2) - 1$ and $\phi(x) = \exp(x) - 1$, giving $(\eps,O(\eps\sqrt{\log(1/\eps)}))$-stability.

\paragraph{Log-concave:} If $X$ is log-concave with bounded variance, then $\|X\|_{\psi_0} = O(1)$ for $\psi_0(x) = \exp(x) - 1$, by Borell's lemma. There is a very slight non-convexity to $\exp(\sqrt{x}) - 1$, which we remedy by taking
\begin{equation*}
    \phi(x) \defeq \begin{cases}
        \exp(\sqrt{x}) - 1, & x \geq 1\\
        e/2 + e|x|/2 - 1, & 0 \leq x < 1
    \end{cases}
\end{equation*}
and $\psi(x) = \phi(x^2)$. We still have $\|X\|_\psi = O(1)$, giving $(\eps,O(\eps \log(1/\eps))$-stability.

\paragraph{Bounded $q$th moments, $q \geq 2$:} We apply the lemma with $\psi(x) = x^{q}$, giving $(\eps,O(\eps^{1-1/q}))$-stability.
\jmlrQED

\subsection{Proof of \cref{lem:W1-decomposition}}
\label{prf:W1-decomposition}

We prove a stronger result.

\begin{lemma}
\label[lemma]{lem:Wp-decomposition}
Fix $0 < \tau < 1$, $1 \leq p < q$, and $P,Q \in \cP(\R^d)$ such that $\Wp(P,Q) \leq \rho$. Then there exists $R \in \cP(\R^d)$ such that $\Wp(P,R) \leq \rho$, $\Wq(P,R) \leq \bigl(\frac{q}{q-p}\bigr)^{1/q} \rho \tau^{1/q - 1/p}$, and $\|R - Q\|_\tv \leq \tau$.
\end{lemma}
\begin{proof}\ \ \ \  
Let $(X,Y)$ be an optimal coupling for $\Wp(P,Q)$, and let $\Delta = Y-X$. Writing $\tau$ for the $1-\tau$ quantile of $\|\Delta\|$, let $E$ denote the event that $\|\Delta\| \leq \tau$. By Markov's inequality, we have
\begin{equation*}
    \tau = \Pr(\|\Delta\| > \tau) = \Pr(\|\Delta\|^p > \tau^p) \leq \frac{\rho^p}{\tau^p},
\end{equation*}
and so $\tau \leq \rho \tau^{-1/p}$. Consider the random variable $Z$ which equals $Y$ if $\|\Delta\| \leq \tau$ and $X$ otherwise. Taking $R$ to be the law of $Z$, we have $\|R - Q\|_\tv \leq \tau$, 
\begin{equation*}
    \Wp(P,R)^p \leq \E[\|X - Z\|^p] \leq \E[\|X - Y\|^p] = \rho^p,
\end{equation*}
and
\begin{align*}
    \Wq(P,R)^q &\leq \E[\|X - Z\|^q]\\
    &\leq \E\bigl[\|\Delta\|^q \bigm| E\bigr]\\
    &\leq \int_0^{\tau^q} \Pr\left[\|\Delta\|^q > t \bigm| E\right] \dd t\\
    &= \int_0^{\tau^q} \Pr\left[\|\Delta\|^p > t^{p/q} \bigm| E\right] \dd t\\
    &\leq \E[\|\Delta\|^p | E] \int_0^{\tau^q} t^{-p/q} \dd t\\
    &\leq \rho^p \frac{t^{1-p/q}}{1-p/q}\big|^{\tau^q}_0\\
    &= \frac{q}{q-p} \rho^p \tau^{q-p}\\
    &\leq \frac{q}{q-p} \rho^p (\rho \tau^{-1/p})^{q-p}\\
    &= \frac{q}{q-p} \rho^q \tau^{1 - q/p}.
\end{align*}
Taking $q$th roots gives the lemma.
\end{proof}

As a consequence, we also have the following.
\begin{lemma}
\label[lemma]{lem:RWp-decomposition-old}
Fix $0 < \eps < 1/2$, integers $p,q \geq 1$, and $P,Q \in \cP(\R^d)$ such that $\RWp(P,Q) \leq \rho$. Then there exists $R \in \cP(\R^d)$ such that $\Wp(P,R) \leq \rho$, $\Wq(P,R) \leq \sqrt{2} \rho \eps^{-[1/p - 1/q]_+}$, and $\|R - Q\|_\tv \leq 2\eps$.
\end{lemma}
\begin{proof}\ \ \ \  
By the $\RWp$ bound, there exists $P' \in \cP(\R^d)$ such that $\Wp(P,P') \leq \rho$ and $\|P' - \tilde{P}\|_\tv \leq \eps$. If $q \leq p$, then we can simply take $R = P'$. Otherwise, $q \geq p + 1$, and we can apply \cref{lem:Wp-decomposition} between $P$ and $P'$ to obtain $R$ such that $\Wq(P,R) \leq \bigl(\frac{q}{q-p}\bigr)^{1/q} \rho \eps^{1/q - 1/p} \leq \sqrt{2} \rho \eps^{-[1/p - 1/q]_+}$ and $\|R - Q\|_\tv \leq \|R - P'\|_\tv + \|P' + \tilde{P}\|_\tv \leq 2\eps$.   
\end{proof}

\subsection{Proof of \cref{lem:TV-risk-bound}}
\label{prf:TV-risk-bound}

Defining $\cS(\eps,\delta,\lambda) \defeq \{ P \in \cP(\R^d) : \Wtwo(P,\cS(\eps,\delta)) \leq \lambda \}$, we prove a stronger statement.

\begin{lemma}
\label{lem:TV-risk-bound-detailed}
Fix $0 \leq \eps < 1/2$ and $\cG \subseteq \cS(2\eps,\delta,\lambda)$. Then, for all $1 \leq q < 2$ and $k \in [d]$, we have $\cR_{\Wqk}(\cG,\|\cdot\|_\tv \leq \eps) \lesssim \frac{1}{1-2\eps}( \sqrt{k}\eps^{1/q-1}\delta + \lambda\eps^{1/q - 1/2})$, for all $k \in [d]$. This risk is achieved by the minimum distance estimator $\mathsf{T}$ satisfying
$\mathsf{T}(\tilde{P}) \in \argmin_{Q \leq \frac{1}{1-\eps} \tilde{P}} \Wtwo(Q,\cG)$, where $\Wtwo(Q,\cG) \defeq \inf_{R \in \cG}\Wtwo(Q,R)$.
\end{lemma}

\begin{proof}\ \ \ \  
Given $P \in \cG$ and $\tilde{P}$ such that $\|\tilde{P} - P\|_\tv \leq \eps$, we decompose $\tilde{P} = (1-\eps)P' + \eps R$ for distribution $P'$ and $R$ such that $P' \leq \frac{1}{1-\eps}P$. Observe that the estimate $\hat{P} = \mathsf{T}(\tilde{P})$ satisfies
\begin{align*}
    \Wtwo(\hat{P},\cG) &= \min_{Q \leq \frac{1}{1-\eps}\tilde{p}} \Wtwo(Q,\cG)\\
    &\leq \Wtwo(P',\cG)\\
    &\leq \Wtwo(P',P)\\
    &\leq \frac{12\eps^{\frac{1}{q}-1}\delta \sqrt{d}}{\sqrt{1-\eps}} + \frac{21 \eps^{\frac{1}{q}-\frac{1}{2}}\lambda}{\sqrt{1-\eps}}. \tag{\cref{lem:Wq-resilience-under-stability}}
\end{align*}
Thus, $\hat{P} \in \cS(2\eps,\delta,\lambda')$, where 
\begin{equation*}
    \lambda' \defeq \frac{12\eps^{\frac{1}{q}-1}\delta \sqrt{d}}{\sqrt{1-\eps}} + \frac{21 \eps^{\frac{1}{q}-\frac{1}{2}}\lambda}{\sqrt{1-\eps}} + \lambda \leq \frac{12\eps^{\frac{1}{q}-1}\delta \sqrt{d}}{\sqrt{1-\eps}} + \frac{22 \lambda}{\sqrt{1-\eps}}.
\end{equation*}
Of course $P \in \cS(2\eps,\delta,\lambda')$ as well, and $\|\hat{P} - P\|_\tv \leq \|\hat{P} - \tilde{P}\|_\tv + \|\tilde{P} - P\|_\tv \leq 2\eps$. Defining the midpoint distribution $Q = (1-\|\hat{P} - P\|_\tv)^{-1} \hat{P} \land P$, we again apply \cref{lem:Wq-resilience-under-stability} to bound
\begin{align*}
    \Wq(\hat{P},P) &\leq \Wq(\hat{P},Q) + \Wq(Q,P)\\
    &\leq \frac{24(2\eps)^{\frac{1}{q}-1}\delta \sqrt{d}}{\sqrt{1-2\eps}} + \frac{42 (2\eps)^{\frac{1}{q}-\frac{1}{2}}\lambda'}{\sqrt{1-2\eps}}\\
    &\leq \frac{528\eps^{\frac{1}{q}-1}\delta \sqrt{d}}{1-2\eps} + \frac{924 \eps^{\frac{1}{q}-\frac{1}{2}}\lambda}{1-2\eps},
\end{align*}
as desired. For $k < d$, we observe that, for each orthogonal projection $U \in \R^{k \times d}$, $UU^\top = I_k$, we still have $\|U_\sharp P - U_\sharp \hat{P}\|_\tv \leq 2\eps$ and $U_\sharp P, U_\sharp \hat{P} \in \cS(2\eps,\delta,\lambda')$, so the analysis above can be applied in $\R^k$ to give the desired bound under $\Wqk$.
\end{proof}

\subsection{Proof of \cref{cor:pop-limit-risk-bounds}}
\label{prf:pop-limit-risk-bounds}

Upper bounds follow by \cref{thm:pop-limit-risk-bd}. Matching lower bounds (up to logarithmic factors) when $\rho = 0$ are shown in Theorem 2 of \cite{nietert2023robust}. It is straightforward to raise these lower bounds by the needed term of $\rho$. Indeed, suppose the learner observes $\tilde{P} = \delta_0$ but the adversary flips a fair coin to select $P = \delta_0$ or $P = \delta_x$ with $\|x\| = \rho$. Then, no estimate $\hat{P}$ can incur expected risk less than that of $\hat{P} = \delta_{\rho/2}$, for which $\Wonek(\hat{P},P) = \rho/2$. (See proof of Theorem 2 in \cite{nietert2023outlier} for a more formal treatment of this two point method).

\jmlrQED

\section{Proofs for \cref{ssec:distribution-learning-finite-sample}}
\label{app:distribution-learning-finite-sample}

Throughout this section, we prove results under a more general learning environment.

\begin{tcolorbox}[colback=white]
\label{setting:B2}
\textbf{Setting B2:}
Fix $\eps > 0$ sufficiently small, $\rho \geq 0$, and $p \in \{1,2\}$. Nature selects a distribution $P \in \Gcov$ and produces $P'$ such that $\Wp(P',P) \leq \rho$ and $\Wtwo(P',P) \leq \rho \eps^{1/2 - 1/p}$. The learner observes $\tilde{P}$ such that $\|\tilde{P} - P'\|_\tv \leq \eps$. All of $P$, $P'$, and $\tilde{P}$ are uniform discrete measures over $n$ points in $\R^d$.
\end{tcolorbox}

The $\Wtwo$ bound is without loss of generality by \cref{lem:W1-decomposition} (no such decomposition is necessary when $p=2$, where we may simply take $P' = P$).

\subsection{Proof of \cref{lem:W2-trace-norm-comparison}}
\label{prf:W2-trace-norm-comparison}

We first prove that $\Wtwo(Q,\Gcov)^2 \leq \tr(\Sigma_Q -  I_d)_+$. Assume without loss of generality that $Q$ has mean 0. Write $\lambda_1 \geq \dots \geq \lambda_d \geq 0$ for the eigenvalues of $\Sigma_Q$, with accompanying eigenvectors $v_1, \dots, v_d \in \R^d$. Define $A = \sum_{i=1}^d \bigl(1 \land (1/\sqrt{\lambda_i}\,)\bigr) v_i v_i^\top$, so that $R = A_\sharp Q$ satisfies $\Sigma_{R} \preceq I_d$. Taking $X \sim Q$, we bound
\begin{align*}
    \Wtwo(Q,\Gcov)^2 &\leq \Wtwo(Q,R)^2\\
    &\leq \E[\|X - AX\|^2]\\
    &= \E[\|(I - A)X\|^2]\\
    &= \tr((I-A)^2\Sigma_Q)\\
    &= \sum_{i : \lambda_i > 1} \bigl(\sqrt{\lambda_i}-1\bigr)^2\\
    &\leq \sum_{i : \lambda_i > 1}\lambda_i -1\\
    &= \tr(\Sigma_\mu - I_d)_+,
\end{align*}
as desired.\medskip

Next, we prove that $\tr(\Sigma_Q - 2 I_d)_+ \leq 2\Wtwo\bigl(Q,\Gcov\bigr)^2$. Suppose that $\Wtwo(Q,R) \leq \lambda$ for some $R \in \Gcov$. Assume without loss of generality that $R$ has mean 0, and fix any $\R^{d \times d}$ with $\|A\|_\op \leq 1$. Taking $(AX,AY)$ to be an optimal coupling for the $\Wtwo(A_\sharp Q,A_\sharp R)$ problem, so that $(X,Y)$ is a coupling of $Q$ and $R$, we have
\begin{align*}
    \tr(A^\top A \Sigma_Q) &=
    \E[\|A(X - \E[X])\|^2]\\
    &\leq \E[\|AX\|^2]\\
    &= \E[\|AY + A(X-Y)\|^2]\\
    &\leq \left(\sqrt{\E[\|AY\|^2]} + \sqrt{\E[\|A(X-Y)\|^2]}\right)^2 \tag{Minkowski's inequality}\\
    &= \left(\sqrt{\tr(A^\top A\Sigma_R)} + \Wtwo(A_\sharp Q,A_\sharp R)\right)^2\\
    &\leq \left(\sqrt{\tr(A^\top A\Sigma_R)} + \lambda \right)^2 \tag{$\|A\|_\op \leq 1$}\\
    &\leq \tr(A^\top A(2\Sigma_R)) + 2\lambda^2\\
    &\leq \tr(A^\top A 2 I_d) + 2\lambda^2 \tag{$R \in \Gcov(\sigma)$}.
\end{align*}
Rearranging and supremizing over $A$, we find that $\tr(\Sigma_Q - 2I_d)_+ \leq 2 \lambda^2$, as desired.\jmlrQED

\subsection{Proof of \cref{thm:W2-project}}
\label{prf:W2-project}

We shall prove a stronger statement under Setting \hyperref[setting:B2]{B2}. We require a slight change to the termination condition when $p = 2$. Specifically, we define the modified algorithm \FilterSimpleTwo{} by changing the termination condition at Step~\ref{step:W2-project-terminate} from \begin{equation*}
    \tr(\Pi(\Sigma_T - \sigma^2I_d)) < C\eps + C\rho^2/\eps
\end{equation*}
to
\begin{equation*}
    \tr(\Pi(\Sigma_T - \sigma^2I_d)) < C\eps + C\rho^2\eps^{1 - 2/p}.
\end{equation*}
For this algorithm (which matches \FilterSimple{} when $p=1$), we prove the following.

\begin{lemma}
\label{lem:W2-project-2}
Under Setting \hyperref[setting:B2]{B2} with $\eps_0 \leq 2^{-20}$, $\FilterSimpleTwo(\tilde{P},\eps,\rho)$ returns $\hat{P}$ in time $\poly(n,d)$ such that, for all $q \in \{1,2\}$ and $k \in [d]$,
\begin{equation*}
    \Wqk(\hat{P},P) \lesssim \eps^{\frac{1}{q}-\frac{1}{2}}\sqrt{k} + \eps^{-\bigl[\frac{1}{p} - \frac{1}{q}\bigr]_+}\rho
\end{equation*}
with probability at least $2/3$.
\end{lemma}

\begin{proof}
First, it is easy to show that the algorithm runs in polynomial time. At least one point is removed at each iteration, so there can be at most $n$ iterations, and each iteration can be performed in $\poly(n,d)$ time.

Next, we show that \FilterSimple{} approximately minimizes its trace norm objective. We begin with a technical claim about the function $f$ computed in Steps~\ref{step:W2-project-g-def}-\ref{step:W2-project-f-def}. This result mirrors Proposition 2.19 of \cite{diakonikolas2023algorithmic}, which pertains to the simpler objective $\lambda_{\max}(\Sigma_Q)$.

\begin{lemma}
\label{lem:W2-project-progress}
Under Setting \hyperref[setting:B2]{B2}, let $S \subseteq \R^d$ denote the support of $P'$. Suppose that, in some iteration of \FilterSimple, the multiset $T$ satisfies $|T \cap S| \geq (1 - 4\eps) |S|$ and $\tr(\Sigma_T - \sigma^2 I_d)_+ \geq 2^{33}\eps + 2^{19}\rho^2\eps^{1 - 2/p}$. Then the function $f$ computed in Steps~\ref{step:W2-project-g-def}-\ref{step:W2-project-f-def} satisfies $\sum_{x \in T} f(x) \geq 2 \sum_{x \in T \cap S} f(x)$.
\end{lemma}
\begin{proof}
Under Setting \hyperref[setting:B2]{B2}, we have $\Wtwo(S,\Gcov) \leq \Wtwo(S,P) \leq \rho \eps^{1/2 - 1/p} \eqqcolon \lambda$. Since $\Gcov \subseteq \Siso(6\eps,5\sqrt{\eps})$, as described in \cref{ex:stability-bounds} (the constant can be obtained from \cref{lem:stability-under-orlicz-norm}), we have that $S \in \cS(6\eps,5\sqrt{\eps},\lambda)$, using the notation from the proof of \cref{lem:TV-risk-bound}. \medskip

Now, at Step~\ref{step:W2-project-g-def}, \FilterSimple{} computes the function $g(x) = \|\Pi(x - \mu_T)\|^2$, where $\Pi$ is the orthogonal projection onto the non-negative eigenspace of $\Sigma_T - \sigma^2I_d$. At Step~\ref{step:W2-project-L-def}, we take $L$ to be the set of $6\eps|T|$ elements of $T$ for which $g(x)$ is largest. Then, Step~\ref{step:W2-project-f-def} takes $f(x) = g(x)$ for $x \in L$ and $f(x) = 0$ otherwise. Define $\eta \defeq \tr(\Sigma_T - \sigma^2I_d)_+ = \tr(\Pi(\Sigma_T - \sigma^2I_d))$, which, by the lemma assumption, satisfies $\eta \geq 2^{33}\eps + 2^{19}\lambda^2$. This lower bound will be used later. We first compute
\begin{align*}
    \sum_{x \in T} g(x) = |T|\tr(\Pi \Sigma_T) = |T|(\eta + \sigma^2\tr(\Pi)).
\end{align*}

Next, we apply \cref{lem:resilience-from-near-stability} to $\Pi_\sharp  S$, noting that $S$, and hence $\Pi_\sharp S$, belongs to $\cS(6\eps,5\sqrt{\eps},\lambda)$. For any $S' \subseteq S$ with $|S'| \geq (1-6\eps)|S|$, this gives
\begin{align*}
    \|\mu_S - \mu_{S'}\| \leq 5\sqrt{\eps} + \frac{2\lambda\sqrt{6\eps}}{1-6\eps} \leq 5\sqrt{\eps} + 10\lambda\sqrt{\eps} \tag{$\eps \leq 1/12$}
\end{align*}
and
\begin{align*}
    |\tr(\Pi\Sigma_{S'}) - \tr(\Pi)| &\leq  \bigl|\tr\bigl(\Pi(\Sigma_{S'} -\Sigma_{S'}(\mu_S))\bigr)\bigr| +  \bigl|\tr\bigl(\Pi(\Sigma_{S'}(\mu_S) - \Sigma_{S})\bigr)\bigr| + \bigl|\tr\bigl(\Pi(\Sigma_S - I_d)\bigr)\bigr|\\
    &\leq \|\mu_S - \mu_{S'}\|^2 +  \left(\frac{100}{1-6\eps}\tr(\Pi) + \frac{16 \lambda^2}{1-6\eps}\right) + \max\left\{\tr(\Pi\Sigma_S),\tr(\Pi)\right\}\\
    &\leq \left(5\sqrt{\eps} + \frac{2\lambda\sqrt{6\eps}}{1-6\eps}\right)^2 +  \left(\frac{100}{1-6\eps}\tr(\Pi) + \frac{16 \lambda^2}{1-6\eps}\right) + 2\tr(\Pi) + 4\lambda^2\\
    &\leq \left(5\sqrt{\eps} + \frac{2\lambda\sqrt{6\eps}}{1-6\eps}\right)^2 +  \left(\frac{100}{1-6\eps}\tr(\Pi) + \frac{16 \lambda^2}{1-6\eps}\right) + 2\tr(\Pi) + 4\lambda^2\\
    &\leq \frac{152}{1-6\eps}\tr(\Pi) + \frac{24\lambda^2}{1-6\eps} \tag{$\eps \leq 1/12$}\\
    &\leq 304\tr(\Pi) + 48\lambda^2.
\end{align*}

Moreover, since $|T| \leq |S|$ and $|T \cap S| \geq (1-6\eps)|S|$, we have that $\|T - S\|_\tv \leq 6\eps$. Using this, that $\tr(\Sigma_T - \sigma^2I_d)_+ = \eta$, and that $S \in \cS(6\eps,5\sqrt{\eps},\lambda)$, we apply \cref{lem:error-certificate-strong} to $T$ and $S$ with $q = k = 1$ and $\eps \gets 6\eps \leq E_1\eps$ to obtain
\begin{align*}
    \|\mu_T - \mu_S\| &\leq \mathsf{W}_{1,1}(T,S)\\
    &\leq \frac{21\left(5\sqrt{\eps} + \sqrt{6\eps}\sigma\right) }{1-6\eps} + \frac{36 \sqrt{\eps}(\lambda + \sqrt{\eta})}{(1-6\eps)^{3/2}}\\
    &\leq 2^{12}\sqrt{\eps} + 102 \sqrt{\eps}(\lambda + \sqrt{\eta}) \tag{$\eps \leq 1/12, \sigma \leq 50$}
\end{align*}
Thus, the triangle inequality gives
\begin{align*}
    \|\mu_T - \mu_{S'}\| &\leq \|\mu_T - \mu_S\| + \|\mu_S - \mu_{S'}\|\\
    &\leq \left(2^{12}\sqrt{\eps} + 102 \sqrt{\eps}(\lambda + \sqrt{\eta})\right) + \left(5\sqrt{\eps} + 10\lambda\sqrt{\eps}\right)\\
    &\leq 2^{13}\sqrt{\eps} + 112\lambda\sqrt{\eps} + 102\sqrt{\eps\eta}
\end{align*}
Combining the above, we have for such $S'$ that
\begin{align*}
    \sum_{x \in S'} g(x) &= |S'| \left(\tr(\Pi \Sigma_{S'}) + \|\Pi(\mu_T - \mu_{S'})\|^2\right)\\
    &\leq |S| \left(\tr(\Pi) + 304\tr(\Pi) + 48\lambda^2 + \left[2^{13} \sqrt{\eps} + 112 \lambda\sqrt{\eps} + 102 \sqrt{\eta\eps}\right]^2\right)\\
    &\leq |S| \left(305\tr(\Pi) + 2^{28}\eps + 2^{14}\lambda^2  + 2^{15} \eta \eps\right)
\end{align*}
and
\begin{align*}
    \sum_{x \in S'} g(x) &= |S'| \left(\tr(\Pi \Sigma_{S'}) + \|\Pi(\mu_T - \mu_{S'})\|^2\right)\\
    &\geq |S'| \left(\tr(\Pi) - 304\tr(\Pi) - 2^{28}\eps - 2^{14}\lambda^2  - 2^{15} \eta \eps\right)\\
    &\geq (1-\eps)|S| \left(-303\tr(\Pi) - 2^{28}\eps - 2^{14}\lambda^2  - 2^{15} \eta \eps\right)\\
    &\geq |S| \left(-303\tr(\Pi) - 2^{28}\eps - 2^{14}\lambda^2  - 2^{15} \eta \eps\right)
\end{align*}
Since $|T| \geq (1 - 4\eps)|S| \geq 2|S|/3$, combining the above gives
\begin{align*}
    \sum_{x \in T \setminus S} g(x) &\geq \sum_{x \in T} g(x) - \sum_{x \in S} g(x)\\
    &= |T|\left(\eta + \sigma^2\tr(\Pi)\right) - |S| \left(305\tr(\Pi) + 2^{28}\eps + 2^{14}\lambda^2  + 2^{15} \eta \eps\right)\\
    &\geq|S|\left(2\eta/3 + 1300\tr(\Pi) - 2^{28}\eps - 2^{14}\lambda^2 - 2^{15}\eta\eps\right) \tag{$\sigma \geq 50$}\\
    &\geq |S|(\eta/4 + 1300\tr(\Pi)) \tag{$\eps \leq 2^{-18}$, $\eta \geq 2^{30}\eps + 2^{14}\lambda^2$}
\end{align*}
Moreover, since $|L| = 6\eps |T| \geq 6\eps |S|(1-4\eps) \geq 4\eps |S|\geq |T \setminus S|$, and $g$ takes its largest values on points of $L$, we have
\begin{equation*}
    \sum_{x \in T} f(x) = \sum_{x \in L} g(x) \geq \sum_{x \in T \setminus S} g(x) \geq |S|(\eta/4 + 1300\tr(\Pi)).
\end{equation*}
Finally, plugging in $S' = S$ and $S' = S \setminus L$ into the bounds above on $\sum_{x \in S'}g(x)$, we obtain
\begin{align*}
    \sum_{x \in S \cap T} f(x) &= \sum_{x \in S \cap L} g(x)\\
    &= \sum_{x \in S} g(x) - \sum_{x \in S \setminus L} g(x)\\
    &\leq |S| \left(609\tr(\Pi) + 2^{29}\eps + 2^{15}\lambda^2  + 2^{16} \eta \eps\right)\\
    &\leq |S| \left(609\tr(\Pi) + \eta/8\right) \tag{$\eps \leq 2^{-20}$, $\eta \geq 2^{33}\eps + 2^{19}\lambda^2$}\\
    &\leq \frac{1}{2} \sum_{x \in T} f(x),
\end{align*}
as desired. 
\end{proof}

Now, by the exact martingale argument used to prove Theorem 2.17 in \cite{diakonikolas2023algorithmic}, \cref{lem:W2-project-progress} implies that \FilterSimple{} maintains the invariant $|S \cap T| \geq (1-4\eps)|S|$ over all iterations with probability at least 2/3. Since at least one point is removed from $T$ at each iteration, the algorithm must terminate while satisfying this invariant as well as the (updated) termination condition at Step~\ref{step:W2-project-terminate}: $\tr(\Sigma_T - \sigma^2 I_d)_+ < C\eps + C\rho^2\eps^{1 - 2/p}$. Consequently, the returned measure $\hat{P} = \Unif(T)$ satisfies 
\begin{equation*}
    \|\hat{P} - P'\|_\tv \leq 4\eps
\end{equation*}
and
\begin{equation*}
    \tr(\Sigma_{\hat{P}} - \Sigma_{P'} - \sigma^2I_d)_+ \leq \tr(\Sigma_{\hat{P}} - \sigma^2I_d)_+ \leq C\eps + C\rho^2\eps^{1 - 2/p}.
\end{equation*}
Thus, by \cref{lem:error-certificate-strong-follow-up} and the fact that $P \in \cS(\eps,O(\sqrt{\eps}))$, we have
\begin{align*}
    \mathsf{W}_{q,k}(Q,P) &\lesssim \eps^{\frac{1}{q}-\frac{1}{2}} \sqrt{k} + \eps^{-[\frac{1}{p} - \frac{1}{q}]_+}\left(\rho + \eps^{\frac{1}{p} - \frac{1}{2}}\sqrt{C\eps + 2C\rho^2 \eps^{1 - \frac{2}{p}}}\right)\\
    &\lesssim \eps^{\frac{1}{q}-\frac{1}{2}} \sqrt{k} + \eps^{-[\frac{1}{p} - \frac{1}{q}]_+}\left(\rho + \eps^{\frac{1}{p}}\right)\\
    &\lesssim \eps^{\frac{1}{q}-\frac{1}{2}} \sqrt{k} + \eps^{-[\frac{1}{p} - \frac{1}{q}]_+}\rho,
\end{align*}
as desired.
\end{proof}

\subsection{Error Certificate Lemmas}
\label{prf:error-certificate}

We state a useful technical lemma extending the certificate lemma (Lemma 2.6) of \cite{diakonikolas2023algorithmic}. Note that here the name is less precise; since $P'$ is not observed, we cannot certify this condition from our observation unless we approximate $\Sigma_{P'}$ by $I_d$.

\begin{lemma}
\label[lemma]{lem:error-certificate-strong}
Let $\lambda_1, \lambda_2, \eta \geq 0$, $\eps \in (0,1)$, and $\delta \geq \eps$. Fix $P' \in \cS(\eps,\delta,\eta)$ and $Q \in \cP(\R^d)$ such that $\tr(\Sigma_{Q} - \Sigma_{P'} - \lambda_1 I_d)_+ \leq \lambda_2$ and $\|Q - P'\|_\tv \leq \eps$. Then
\begin{equation*}
    \Wqk(Q,P') \leq \frac{21\eps^{\frac{1}{q}-1}\tilde{\delta} \sqrt{k}}{1-\eps} + \frac{36 \eps^{\frac{1}{q}-\frac{1}{2}}\tilde{\eta}}{(1-\eps)^{3/2}}
\end{equation*}
for all $k \in [d]$, where $\tilde{\delta} = \delta + \sqrt{\lambda_1\eps}$ and $\tilde{\eta} = \eta + \sqrt{\lambda_2}$.
\end{lemma}

We now apply this result under Setting \hyperref[setting:B2]{B2}.

\begin{lemma}
\label[lemma]{lem:error-certificate-strong-follow-up}
Let $\lambda_1, \lambda_2 \geq 0$ and $C \geq 1$. Under Setting \hyperref[setting:B2]{B2} with $P \in \cS(C\eps,\delta)$, fix any $Q \in \cP(\R^d)$ such that $\tr(\Sigma_{Q} - \Sigma_{P'} - \lambda_1 I_d)_+ \leq \lambda_2$ and $\|Q - P'\|_\tv \leq \tau$, where $\eps \leq \tau \leq C\eps$. Then
\begin{equation*}
    \mathsf{W}_{q,k}(Q,P) \leq \frac{21\tau^{\frac{1}{q}-1}\tilde{\delta} \sqrt{k}}{1-\tau} + \frac{37 \tau^{-[1/p - 1/q]_+}\tilde{\rho}}{(1-\tau)^{\frac{3}{2}}}
\end{equation*}
for all $k \in [d]$, where $\tilde{\delta} = \delta + \sqrt{\lambda_1\tau}$ and $\tilde{\rho} = \rho + \tau^{1/p - 1/2}\sqrt{\lambda_2}$.
\end{lemma}
\begin{proof}
Under Setting \hyperref[setting:B2]{B2}, we have $P' \in \cS(C\eps,\delta,\eta) \subseteq \cS(\tau,\delta,\eta)$, where $\eta = \rho \eps^{1/2 - 1/p} \geq \rho \tau^{1/2 - 1/p}$. Applying \cref{lem:error-certificate-strong} to $P'$ and $Q$ with TV corruption level $\tau$ and plugging in our value of $\eta$ gives
\begin{equation*}
    \Wqk(Q,P') \leq \frac{21\tau^{\frac{1}{q}-1}\tilde{\delta} \sqrt{d}}{1-\tau} + \frac{36 \tau^{\frac{1}{q}-\frac{1}{2}}\eta}{(1-\tau)^{3/2}}\\
    \leq \frac{21\tau^{\frac{1}{q}-1}\tilde{\delta} \sqrt{d}}{1-\tau} + \frac{36 \tau^{\frac{1}{q}-\frac{1}{p}}\tilde{\rho}}{(1-\tau)^{3/2}}.
\end{equation*}
Moreover, we have $\Wqk(P,P') \leq \rho \eps^{-[1/p - 1/q]_+} \leq \tilde{\rho} \tau^{-[1/p - 1/q]_+}$. Noting that $\tau^{-[1/p - 1/q]_+} \geq \tau^{1/q - 1/p}$, the lemma follows by the triangle inequality.    
\end{proof}

We now return to the initial technical lemma.
\medskip

\par\noindent{\bfseries\upshape Proof of \cref{lem:error-certificate-strong}\ }%
By the TV bound, we can decompose $Q = (1-\eps)\bar{P} + \eps R$ for some $\bar{P} \leq \frac{1}{1-\eps}P'$. Using this decomposition, we bound
\begin{align*}
    \lambda_2 &\geq \tr(\Sigma_Q - \Sigma_{P'} - \lambda_1 I_d)\\
    &= \tr(\Sigma_{\bar{P}} - \Sigma_{P'}) + \eps \tr(\Sigma_R) + \eps (1-\eps) \|\mu_{\bar{P}} - \mu_R\|^2 - \eps \tr(\Sigma_{\bar{P}}) - d\lambda_1\\
    &\geq \tr(\Sigma_{\bar{P}} - \Sigma_{P'}) + \eps \tr(\Sigma_R) - \eps \tr(\Sigma_{\bar{P}}) - d\lambda_1.
\end{align*}
Since $P' \in \cS(\eps,\delta,\eta)$, we can rearrange the above and apply \cref{lem:resilience-from-near-stability} to obtain
\begin{align*}
    \tr(\Sigma_R) &\leq \frac{\lambda_2}{\eps} + \frac{1}{\eps}\tr(\Sigma_{P'} - \Sigma_{\bar{P}}) + \tr(\Sigma_{\bar{P}}) + \frac{d\lambda_1}{\eps}\\
    &\leq \frac{\lambda_2}{\eps} + \frac{1}{\eps} \tr(\Sigma_{P'} - \Sigma_{\bar{P}}(\mu_{P'})) + \frac{1}{\eps}\|\mu_{P'} - \mu_{\bar{P}}\|^2 + \tr(\Sigma_{\bar{P}}) + \frac{d\lambda_1}{\eps}\\
    &\leq \frac{\lambda_2}{\eps} + \frac{1}{\eps} \tau_2(P',\eps) + \frac{1}{\eps}\tau(P,\eps)^2 + \tr(\Sigma_{\bar{P}}) + \frac{d\lambda_1}{\eps}\\
    &\leq \frac{\lambda_2}{\eps} + \frac{1}{\eps}\left[\frac{4d\delta^2}{(1-\eps)\eps} + \frac{16 \eta^2}{1-\eps}\right] + \frac{1}{\eps}\left[\delta + \frac{2\sqrt{\eps}\eta}{1-\eps}\right]^2 + \left[\frac{6d\delta^2}{(1-\eps)\eps^2} + \frac{20\eta^2}{1-\eps}\right] + \frac{d\lambda_1}{\eps}\\
    &\leq \frac{d\lambda_1}{\eps} + \frac{\lambda_2}{\eps} + \frac{12d\delta^2}{(1-\eps)\eps^2} + \frac{44\eta^2}{(1-\eps)^2\eps}\\
    &\leq \frac{13d\tilde{\delta}^2}{(1-\eps)\eps^2} + \frac{45\tilde{\eta}^2}{(1-\eps)^2\eps}.
\end{align*}

Next, take $v$ to be the unit vector in the direction of $\mu_{P'} - \mu_R$. Then, applying a similar argument as above, we bound
\begin{align*}
    \lambda_2 &\geq v^\top\! (\Sigma_Q - \Sigma_{P'} - \lambda_1 I_d) v\\
    &= v^\top\!(\Sigma_{\bar{P}} - \Sigma_{P'})v + 2\eps v^\top \Sigma_R v + \eps(1-\eps) \|\mu_{\bar{P}} - \mu_R\|^2 - \eps v^\top\Sigma_{\bar{P}}v - \lambda_1\|v\|^2\\
    &\geq v^\top\!(\Sigma_{\bar{P}} - \Sigma_{P'})v + \eps(1-\eps) \|\mu_{\bar{P}} - \mu_R\|^2 - \eps v^\top\Sigma_{\bar{P}}v - \lambda_1
\end{align*}
Rearranging and applying \cref{lem:resilience-from-near-stability} to $v^\top_\sharp P'$, we bound $\eps(1-\eps)\|\mu_{\bar{P}} - \mu_R\|^2$ by
\begin{align*}
    &\lambda_2 + v^\top\!(\Sigma_{P'} - \Sigma_{\bar{P}})v + \eps v^\top\Sigma_{\bar{P}}v + \lambda_1\\
    \leq\; &\lambda_2 + v^\top\left(\Sigma_{P'} - \Sigma_{\bar{P}}(\mu_{P'})\right)v + \|\mu_{P'} - \mu_{\bar{P}}\|^2  + \eps v^\top\Sigma_{\bar{P}}v + \lambda_1\\
    \leq\; &\lambda_2 + \left(\frac{4\delta^2}{(1-\eps)\eps} + \frac{16 \eta^2}{1-\eps}\right) + \left(\delta + \frac{2\sqrt{\eps}\eta}{1-\eps}\right)^2  + \eps \left(\frac{6\delta^2}{(1-\eps)\eps^2} + \frac{20\eta^2}{1-\eps}\right) + \lambda_1\\
    \leq\; &\lambda_1 + \lambda_2 + \frac{12\delta^2}{(1-\eps)\eps} + \frac{44\eta^2}{(1-\eps)^2}\\
    \leq\; &\frac{13\tilde{\delta}^2}{(1-\eps)\eps} + \frac{45\tilde{\eta}^2}{(1-\eps)^2}
\end{align*}
We thus have
\begin{equation*}
    \|\mu_{\bar{P}} - \mu_R\|^2 \leq \frac{13\tilde{\delta}^2}{(1-\eps)^2\eps^2} + \frac{45\tilde{\eta}^2}{(1-\eps)^3\eps}
\end{equation*}
Combining with an application of \cref{lem:resilience-from-near-stability} to $P'$, we bound
\begin{align*}
    \|\mu_R - \mu_{P'}\|^2 &\leq 2\|\mu_R - \mu_{\bar{P}}\|^2 + 2\|\mu_{\bar{P}} - \mu_{P'}\|^2\\
    &\leq  \frac{26\tilde{\delta}^2}{(1-\eps)^2\eps^2} + \frac{90\tilde{\eta}^2}{(1-\eps)^3\eps} + 2\left(\delta + \frac{2\sqrt{\eps}\eta}{1-\eps}\right)^2\\
    &\leq \frac{30\tilde{\delta}^2}{(1-\eps)^2\eps^2} + \frac{98\tilde{\eta}^2}{(1-\eps)^3\eps}.
\end{align*}
Next, we apply \cref{lem:tv-regularizer-helper} to bound $\Wq(Q,P')$ by
\begin{align*}
    &\frac{7\eps^{\frac{1}{q}-1}\delta \sqrt{d}}{\sqrt{1-\eps}} + \frac{12 \eps^{\frac{1}{q}-\frac{1}{2}}\eta}{\sqrt{1-\eps}} + 2 \eps^\frac{1}{q} \sqrt{\tr(\Sigma_R) + \|\mu_R - \mu_{P'}\|^2}\\
    \leq\:& \frac{7\eps^{\frac{1}{q}-1}\delta \sqrt{d}}{\sqrt{1-\eps}} + \frac{12 \eps^{\frac{1}{q}-\frac{1}{2}}\eta}{\sqrt{1-\eps}} + 2 \eps^\frac{1}{q} \sqrt{\frac{13d\tilde{\delta}^2}{(1-\eps)\eps^2} + \frac{45\tilde{\eta}^2}{(1-\eps)^2\eps} + \frac{30\tilde{\delta}^2}{(1-\eps)^2\eps^2} + \frac{98\tilde{\eta}^2}{(1-\eps)^3\eps}}\\
    \leq\:& \frac{7\eps^{\frac{1}{q}-1}\tilde{\delta} \sqrt{d}}{\sqrt{1-\eps}} + \frac{12 \eps^{\frac{1}{q}-\frac{1}{2}}\tilde{\eta}}{\sqrt{1-\eps}} + 2 \eps^\frac{1}{q} \sqrt{\frac{43\tilde{\delta}^2}{(1-\eps)^2\eps^2} + \frac{143\tilde{\eta}^2}{(1-\eps)^3\eps}}\\
    \leq\:& \frac{21\eps^{\frac{1}{q}-1}\tilde{\delta} \sqrt{d}}{1-\eps} + \frac{36 \eps^{\frac{1}{q}-\frac{1}{2}}\tilde{\eta}}{(1-\eps)^{3/2}}
\end{align*}
as desired. As usual, for $k < d$, we note that the analysis above applies to all $k$-dimensional orthogonal projections of the input measures, with the substitution $d \gets k$.
\jmlrQED

\subsection{Proof of \cref{cor:W2-project-statistical}}
\label{prf:W2-project-statistical}

Fix $0 \leq \delta < 1$. By Theorem 2.7 in \citep{boedihardjo2024sharp}, we have
\begin{align}
\label{eq:empirical-W1k-error-bd}
    \E[\Wonek(P,P_n)] \lesssim \frac{k}{n}\E\left\|\sum_{i=1}^n g_i X_i\right\| + k\E\left[\left(\frac{1}{n} \sup_{v \in \unitsph} \sum_{i=1}^n \bigl|\langle v, X_i\rangle\bigr|^{2+2\delta}\right)^{\frac{1}{2+2\delta}}\right] \Phi(n,k,\delta),
\end{align}
where $g_1, \dots, g_n$ and $X_1, \dots, X_n$ are sampled i.i.d.\ from $\cN(0,I_d)$ and $P$, respectively, and $\Phi(n,k,\delta) = \tilde{O}((\delta n)^{-1/(k \lor 2)})$. Since $\Sigma_P \preceq I_d$, we have
\begin{align*}
    \E\left\|\sum_{i=1}^n g_i X_i\right\| \leq \left(\E\left\|\sum_{i=1}^n g_i X_i\right\|^2\right)^\frac{1}{2} = \left(\sum_{i=1}^n \E\|X_i\|^2\right)^\frac{1}{2} \leq \sqrt{nd}.
\end{align*}
Thus, the first term of \eqref{eq:empirical-W1k-error-bd} is at most $k\sqrt{d/n}$. We now set $\delta = 1/\log(n)$. In this case, $n^{-\frac{1}{2+2\delta}} \lesssim n^{-1/2}$, and so
\begin{align*}
    \E\left[\left(\frac{1}{n} \sup_{v \in \unitsph} \sum_{i=1}^n \bigl|\langle v, X_i\rangle\bigr|^{2+2\delta}\right)^{\frac{1}{2+2\delta}}\right] &\leq \E\left[\left(\frac{1}{n} \sum_{i=1}^n \|X_i\|^{2+2\delta}\right)^{\frac{1}{2+2\delta}}\right]\\
    &\lesssim n^{-1/2} \E\left[\left(\sum_{i=1}^n \|X_i\|^{2+2\delta}\right)^{\frac{1}{2+2\delta}}\right]\\
    &\leq n^{-1/2} \E\left[\left(\sum_{i=1}^n \|X_i\|^{2}\right)^{\frac{1}{2}}\right]\\
    &\leq \sqrt{d}.
\end{align*}
Combining, we obtain that $\E[\Wonek(P,P_n)] = \tilde{O}(k\sqrt{d}n^{-1/(k \lor 2)})$. To obtain a uniform bound over $k \in [d]$ with high (constant) probability, we employ Markov's inequality for all $k$ which are a multiple of $(1 + 1/\log(n))$ when rounded up or down to the nearest integer. There are at most $O(\log(d)\log(n))$ such $k$, and $n^{(1 + 1/\log(n))^{-1}} \gtrsim n/\log(n)$, so a union bound gives that $\max_k\Wonek(P,P_n) = \tilde{O}(\sqrt{d}kn^{-1/(k \lor 2)})$ with high constant probability, as desired.

\jmlrQED

\subsection{Proof of \cref{prop:rho-zero}}
\label{prf:rho-zero}

We have in this case that $\Sigma_{P'} = \Sigma_P$ (since $\rho = 0$). Then, if $\|\Sigma_{\hat{P}}\|_\mathrm{op} \leq 1 + C\delta^2/\eps$, we have by stability of $P$ that
\begin{equation*}
    \tr(\Sigma_{\hat{P}} - \Sigma_P - (1 + (C-1)\tfrac{\delta^2}{\eps})I_d)_+ \leq \tr(\Sigma_{\hat{P}} - (1 + C\tfrac{\delta^2}{\eps})I_d)_+ \leq 0,
\end{equation*}
at which point \cref{lem:error-certificate-strong-follow-up} with $p = q = 1$ gives the Proposition.

\end{document}